\PassOptionsToPackage{square,comma,numbers,sort&compress,super}{natbib}
\documentclass{article}

\usepackage[final]{neurips_2020}

\usepackage[utf8]{inputenc} % allow utf-8 input
\usepackage[T1]{fontenc}    % use 8-bit T1 fonts
\usepackage{hyperref}       % hyperlinks
\usepackage{url}            % simple URL typesetting
\usepackage{booktabs}       % professional-quality tables
\usepackage{amsfonts}       % blackboard math symbols
\usepackage{nicefrac}       % compact symbols for 1/2, etc.
\usepackage{microtype}      % microtypography
\usepackage{microtype}
\usepackage{graphicx}
\usepackage{subfigure}
\usepackage{booktabs} % for professional tables
\usepackage{amsmath}
\usepackage{amssymb}

\usepackage{hyperref}
\usepackage{amsmath,amsfonts,bm}

\def\eqref#1{Equation~(\ref{#1})}
\def\1{\bm{1}}

\def\vv{{\bm{v}}}

\def\mI{{\bm{I}}}

\DeclareMathAlphabet{\mathsfit}{\encodingdefault}{\sfdefault}{m}{sl}
\SetMathAlphabet{\mathsfit}{bold}{\encodingdefault}{\sfdefault}{bx}{n}

\newcommand{\E}{\mathbb{E}}

\newcommand{\R}{\mathbb{R}}

\usepackage{url}
\usepackage{xcolor}
\usepackage[utf8]{inputenc}
\usepackage[english]{babel}
\usepackage{amsthm}
\usepackage{graphicx}
\usepackage{mathrsfs}
\usepackage{color}
\usepackage{wrapfig}

\usepackage[algo2e]{algorithm2e}
\usepackage{algorithm}

\usepackage{float}

\newtheorem{theorem}{Theorem}
\newtheorem{corollary}{Corollary}
\newtheorem*{theorem*}{Theorem}
\newtheorem{lemma}{Lemma}

\newtheorem{proposition}{Proposition}

\title{Black-Box Certification with Randomized Smoothing:\\ A Functional Optimization Based Framework}

\author{%
  Dinghuai Zhang$^*$ \\
  Mila\\
  \texttt{dinghuai.zhang@mila.quebec} \\
   \And
  {Mao Ye}$^*$, {Chengyue Gong}\thanks{Equal contributions} \\
  Department of Computer Science \\
   University of Texas at Austin   \\
   \texttt{\{my21, cygong\}@cs.utexas.edu} \\
   \AND
   Zhanxing Zhu \\
   School of Mathematical Sciences\\
   Peking University \\
   \texttt{zhanxing.zhu@pku.edu.cn} \\
   \And
   Qiang Liu \\
   Department of Computer Science \\
   University of Texas at Austin \\
   \texttt{lqiang@cs.utexas.edu} \\
}

\begin{document}

\maketitle
\newcommand{\fix}{\marginpar{FIX}}
\newcommand{\new}{\marginpar{NEW}}

\global\long\def\u{\boldsymbol{\mu}}%
\global\long\def\x{\boldsymbol{x}}%
\global\long\def\E{\mathbb{E}}%
\global\long\def\0{\boldsymbol{0}}%
\global\long\def\z{\boldsymbol{z}}%
\global\long\def\R{\mathbb{R}}%
\global\long\def\d{\boldsymbol{\delta}}%
\global\long\def\ff{f_{\pi_{\0}}^{\sharp}}%

\newcommand{\functional}{\text{functional~}}
\def\viz{\emph{viz}\onedot}

\newcommand{\ftrue}{f^\sharp}
\newcommand{\ft}{\ftrue}
\newcommand{\xt}{\boldsymbol{x}_0}
\newcommand{\xtrue}{\boldsymbol{x}_0}
\renewcommand{\x}{\boldsymbol{x}_0}
\renewcommand{\z}{\boldsymbol{z}}
\newcommand{\conv}{*}
\newcommand{\RR}{\R}
\renewcommand{\vec}[1]{\boldsymbol{#1}}
\renewcommand{\vv}[1]{\boldsymbol{#1}}
\newcommand{\red}[1]{\textcolor{red}{#1}}
\newcommand{\pizero}{\pi_{\boldsymbol{0}}}
\newcommand{\fpizero}{{f_{\pizero}^\sharp}}
\newcommand{\normal}{\mathcal{N}}
\newcommand\norm[1]{\left\lVert#1\right\rVert}
\newcommand{\F}{\mathcal{F}}
\newcommand{\B}{\mathcal{B}}
\newcommand{\D}{\mathbb{D}}
\newcommand{\V}{\mathcal{L}_{\pizero}}
\newcommand{\laplace}{\mathrm{Laplace}}

\newcommand{\ant}[1]{~~~~~~~\text{\med{//#1}}}
\newcommand{\Qiang}[1]{\textcolor{red}{#1}\\}
\newcommand{\qiang}[1]{\textcolor{red}{#1}}
\newcommand{\Todo}[1]{\textcolor{red}{TODO: #1}\\}
\newcommand{\todo}[1]{\textcolor{gray}{TODO: #1}}
\newcommand{\med}[1]{\textcolor{magenta}{#1}}
\newcommand{\blue}[1]{\textcolor{blue}{#1}}
\newcommand{\gray}[1]{\textcolor{blue}{#1}}
\newcommand{\green}[1]{\textcolor{green}{#1}}
\newcommand{\myempty}[1]{}

\begin{abstract}

  Randomized classifiers have been shown to provide a promising approach for achieving certified robustness against adversarial attacks in deep learning. However, most existing methods only leverage Gaussian smoothing noise and only work for $\ell_2$ perturbation. We propose a general framework of adversarial certification with non-Gaussian noise and for more general types of attacks, from a unified functional optimization perspective. Our new framework allows us to identify a key trade-off  between accuracy and robustness via designing smoothing distributions and leverage it to design new families of non-Gaussian smoothing distributions that work more efficiently for different $\ell_p$ settings, including $\ell_1$, $\ell_2$ and $\ell_\infty$ attacks. Our proposed methods achieve better certification results than previous works and provide a new perspective on randomized smoothing certification.
\end{abstract}

\section{Introduction}
% Deep neural networks have achieved state-of-the-art performance on many tasks such as image classification \cite{he2016deep} and language modeling \cite{DBLP:conf/naacl/DevlinCLT19}. Nonetheless, modern deep learning models are highly sensitive to small and adversarially crafted perturbations on the inputs \cite{goodfellow2014explaining}, which means that human-imperceptible changes on inputs could cause the model to make dramatically different predictions. 
Although many robust training algorithms have been developed to overcome adversarial attacks \cite{kannan2018adversarial, zhang2019defense, zhai2019adversarially}, most heuristically developed methods can be shown to be broken by more powerful adversaries eventually (\emph{e.g.,} \cite{athalye2018obfuscated, madry2017towards, zhang2019you, wang2019improving}). 
This casts an urgent demand for developing robust classifiers with provable worst-case guarantees. One promising approach for certifiable robustness is the recent \emph{randomized smoothing method}  \cite{lecuyer2018certified, cohen2019certified, salman2019provably, lee2019stratified, li2018second, dvijotham2020a, teng2020ell, jia2020certified}, which constructs smoothed classifiers with certifiable robustness by introducing noise on the inputs. Compared with the other more traditional certification approaches \cite{wong2017provable, DBLP:conf/uai/DvijothamSGMK18, jordan2019provable} that exploit special structures of the neural networks (such as the properties of ReLU), the randomized smoothing 
approaches work more flexibly on general black-box classifiers and is shown to be more scalable and provide tighter bounds on challenging datasets such as ImageNet~\cite{deng2009imagenet}.  

Most existing  methods use Gaussian noise for smoothing. Although appearing to be a natural choice, one of our key observations is that the Gaussian distribution is, in fact a sub-optimal choice in high dimensional spaces even for $\ell_2$ attack. We observe that there is a counter-intuitive phenomenon in high dimensional spaces \cite{vershynin2018high}, that almost all of the probability mass of standard Gaussian distribution concentrates around the sphere surface of a certain radius. This makes tuning the variance of Gaussian distribution an inefficient way to trade off robustness and accuracy for randomized smoothing.
% This is due to a counter-intuitive phenomenon in high dimensional spaces \cite{vershynin2018high} that almost all of the probability mass of standard Gaussian  distribution concentrates around the sphere of radius one (and hence ``soap bubble'' in the title), instead of the center point (which corresponds to the original input). As a result, 
% the variance of the Gaussian noise needs to be sufficiently small 
% to yield a good approximation to the original classifiers (by squeezing the ``soap bubble'' towards the center point), 
% which, however,  makes it difficult to verify due to the small noise.
% Further, for the more challenging $\ell_\infty$ attack, Gaussian smoothing provably degenerates in high dimensions. 

\textbf{Our Contributions}\space\space
We propose a general framework of adversarial certification using non-Gaussian smoothing noises, based on a new \functional optimization perspective. 
Our framework unifies the methods of \cite{cohen2019certified} and \cite{teng2020ell} as special cases, and is applicable to more general smoothing distributions and more types of attacks beyond $\ell_2$-norm setting. 
% Importantly, our new framework reveals 
% a trade-off between accuracy and robustness for guiding better choices of smoothing distributions. 
Leveraging our insight, 
we develop a new family of distributions for better certification results on $\ell_1$, $\ell_2$ and $\ell_\infty$ attacks.
An efficient computational approach is developed to enable our method in practice. 
Empirical results show that our new framework and smoothing distributions  outperform existing approaches for $\ell_1$, $\ell_2$ and $\ell_\infty$ attacking, on datasets such as CIFAR-10 and ImageNet. 
\section{Related Works}

% \textbf{Empirical Defenses}\space\space Since \cite{szegedy2013intriguing} and \cite{goodfellow2014explaining}, many previous works have focused on utilizing small perturbation $\delta$ under certain constraint, \emph{e.g.} in a $\ell_p$ norm ball, to attack a neural network. Adversarial training \cite{madry2017towards} and its variants  \cite{kannan2018adversarial, zhang2019defense, zhai2019adversarially} are the most successful defense methods. However, these empirical defense methods are still easy to be broken and cannot provide provable defenses.

\textbf{Certified Defenses}\space\space
Unlike the empirical defense methods,
once a classifier can guarantee a consistent prediction for input within a local region, 
it is called a certified-robustness classifier.
% \subsection{Adversarial Certification}
% A robust certificate ensures a constant prediction under a given perturbation region.
\emph{Exact} certification methods provide the minimal perturbation condition which leads to a different classification result. 
This line of work focuses on deep neural networks with ReLU-like activation that makes the classifier a piece-wise linear function.
This enables researchers to introduce satisfiability modulo theories \cite{carlini2017provably, ehlers2017formal} 
or mix integer linear programming \cite{cheng2017maximum, dutta2017output}. 
\emph{Sufficient} certification methods take a conservative way and bound the Lipschitz constant or other information of the network \cite{jordan2019provable, wong2017provable, raghunathan2018semidefinite, zhang2018efficient}.
However, 
these certification strategies share a drawback that they are not feasible on large-scale scenarios, \emph{e.g.} large and deep networks and datasets.

\textbf{Randomized Smoothing}\space\space
To mitigate this limitation of previous certifiable defenses, 
improving network robustness via randomness has been recently 
discussed \cite{xie2017mitigating, liu2018towards}.
% In certification community,
\cite{lecuyer2018certified} first introduced randomization with technique in differential privacy. 
\cite{li2018second} improved their work 
with a bound given by R\'enyi divergence.
% using method inspired from information theory. 
In succession,
\cite{cohen2019certified} firstly provided a \emph{tight} bound for \emph{arbitrary} Gaussian smoothed classifiers based on previous theorems found by \cite{li1998some}.
% As more robust classifiers can give better certification performance,
\cite{salman2019provably} combined the empirical and certification robustness, 
by applying adversarial training on randomized smoothed classifiers to achieve a higher certified accuracy.
\cite{lee2019stratified} focused on $\ell_0$ norm perturbation setting, and proposed a discrete smoothing distribution which can be shown perform better than the widely used Gaussian distribution. \cite{teng2020ell} took a similar statistical testing approach with \cite{cohen2019certified}, utilizing Laplacian smoothing to tackle $\ell_1$ certification problem. \cite{jia2020certified} extended the approach of \cite{cohen2019certified} to a top-k setting. \cite{dvijotham2020a} extends the total variant used by ~\cite{cohen2019certified} to $f$-divergences. 
% Similar to \cite{dvijotham2020a}, 
Recent works \cite{yang2020randomized, blum2020random, kumar2020curse} discuss further problems about certification methods.
We also focus on a generalization of randomized smoothing, but with a different view on loosing the constraint on classifier.

Noticeably, \cite{yang2020randomized} also develops analysis on $\ell_1$ setting and provide a thorough theoretical analysis on many kinds of randomized distribution.
 We believe the \cite{yang2020randomized} and ours have different contributions and were developed concurrently. \cite{yang2020randomized} derives the optimal shapes of level sets for $\ell_p$ attacks based on the Wulff Crystal theory, while our work, based on our functional-optimization framework and accuracy-robustness decomposition (Eq.\ref{eq:objective2}), proposes to use distribution that is more concentrated toward the center.
Besides, we also consider a novel distribution using mixed $\ell_2$ and $\ell_\infty$ norm for $\ell_{\infty}$ adversary, which hasn't been studied before and improve the empirical results.

% \vspace{-5pt}

\section{Black-box Certification as Functional Optimization} % not a good name?
\label{sec:black-box}

% We start from introducing background of the adversarial certification problem and the randomized smoothing method. In Section~\ref{sec:certification}, we propose our general framework for adversarial certification using general smoothing noises, from a new \functional optimization perspective. Our framework unifies the method of \cite{cohen2019certified, teng2020ell} as special cases, and reveals a potential trade-off between accuracy and robustness that provides important guidance for  better choices of smoothing distributions in Section~\ref{sec:filling}. 

% \vspace{-5pt}
\subsection{Background}
%\paragraph{Randomized Smoothed Classifier} 
\paragraph{Adversarial Certification}
For simplicity, we consider binary classification of predicting binary labels $y\in \{0,1\}$ given feature vectors $x\in \RR^d$. The extension to multi-class cases is straightforward, and is discussed in Appendix~\ref{sec:bilateral}. We assume $\ft \colon \RR^d \to [0,1]$ is a given binary classifier ($\sharp$ means the classifier is \emph{given}), which maps from the input space $\RR^d$ to either the positive class probability in interval $[0,1]$ or binary labels in $\{0,1\}$. 
%is a binary classifier.
In the robustness certification problem, a testing data point $\xt \in \RR^d$ is given,  
%and an \emph{adversarial set} $\B$, which is a neighorhood 
and one is asked to verify if the classifier outputs the same prediction when the input $\xt$ is perturbed arbitrarily in $\B$, a given neighborhood of $\xt$.
Specifically, let $\B$ be a set of possible perturbation vectors, \emph{e.g.}, $\B = \{\d \in \RR^d : \norm{\d}_p \leq r\}$ for $\ell_p$ norm with a radius $r$. 
If the classifier predicts $y = 1$ on $\xt$, i.e. $\ftrue(\xt) > 1/2$, we want to verify if $\ftrue(\xt + \vec\delta) > 1/2$ still holds for any $\delta \in \B$. Through this paper, we consider the most common adversarial settings: $\ell_1, \ell_2$ and $\ell_\infty$ attacks.
% two types of attacks, including the $\ell_{2}$ attack $\B_{\ell_2,r}\defeq\left\{ \d:\left\Vert \d\right\Vert _{2}\le r\right\} $, and the $\ell_{\infty}$ attack $\B_{\ell_\infty, r}\defeq\left\{ \d:\left\Vert \d\right\Vert _{\infty}\le r\right\} $. 

\vspace{-5pt}
\paragraph{Black-box Randomized Smoothing Certification}
Directly certifying $\ftrue$ heavily relies on the smooth property of $\ftrue$, which has been explored in a series of prior works \cite{ wong2017provable, lecuyer2018certified}. 
These methods typically depend on the special structure-property (\emph{e.g.}, the use of ReLU units) of $\ftrue$,
and thus can not serve as general-purpose algorithms for any type of networks. 
Instead, We are interested in \emph{black-box} verification methods that could work for \emph{arbitrary} classifiers. 
One approach to enable this, as explored in recent works \cite{cohen2019certified, lee2019stratified}, is to replace $\ftrue$ with a smoothed classifier by convolving it with Gaussian noise, and verify the \emph{smoothed} classifier.

Specifically, assume $\pizero$ is a smoothing distribution with zero mean and bounded variance, \emph{\emph{e.g.}}, $\pizero = \mathcal N(\vec 0, \sigma^2)$. 
The randomized smoothed
classifier is defined by 
\begin{align*} %\label{equ:fsmooth}
\fpizero(\xtrue) := 
\mathbb{E}_{\z\sim\pi_{\boldsymbol{0}}} \left [ \ftrue(\xtrue +\z) \right], % = \fpizero(\xtrue), 
\end{align*}
which returns the averaged probability of $\x+\z$ under the perturbation of $\z \sim \pi_{\boldsymbol{0}}$.
Assume we replace the original classifier with $\fpizero$, then the goal becomes certifying $\fpizero$ using its inherent smoothness. 
Specifically, if $\fpizero(\xtrue) > 1/2$, we want to certify that $\fpizero(\xtrue + \d) > 1/2 $ for every $\d \in \B$, that is, we want to certify that
\begin{align}
\label{equ:p02}
\min_{\d \in \B} \fpizero(\x +\d) = 
\min_{\d \in \B} \E_{\z\sim \pizero}[\ftrue(\x+\z+\d)]  > \frac{1}{2}.
\end{align}
% \vspace{-1mm}
In this case, it is sufficient to obtain a 
\emph{guaranteed lower bound} of $\min_{\d \in \B} \fpizero(\x +\d)$ and check if it is larger than $1/2$. 
When $\pizero$ is Gaussian $\normal(\vec 0, ~ \sigma^2)$ and for $\ell_2$ attack,  this problem was studied in  \cite{cohen2019certified}, which shows that a lower bound of 
\begin{align}
\label{equ:phib}
\min_{\z \in \B} 
\E_{\z\sim \pizero}[\ftrue(\x +\z)] 
\geq \Phi(\Phi^{-1}(\fpizero(\x))  - \frac{r}{\sigma}), 
\end{align}
where $\Phi(\cdot)$ is the cumulative density function (CDF) of standard Gaussian distribution.
% and $\Phi^{-1}(\cdot)$ represents its inverse cumulative distribution function. 
The proof of this result in \cite{cohen2019certified} 
uses Neyman-Pearson lemma \cite{li1998some}. In the following section, we will show that this bound is a special case of the proposed functional optimization framework for robustness certification.
\subsection{Constrained Adversarial Certification}
\label{sec:certification}

%One approach to the certification problem above is to calculate a \emph{guaranteed lower bound} of $E(\pizero, \ftrue, \B)$ and check if it is larger than $0.5$. 
%The key challenge is that the optimization on $\boldsymbol{\delta}$ is often a challenging non-convex optimization problem.
%A guaranteed lower bound must have a exact worst case $\mu$, which cannot be achieved by most of the optimization methods.
%In other words, directly solving an optimization problem $\min_{\d \in \mathcal{B}} f(\d)$ can only yield an upper bound, instead of \emph{guaranteed} lower bound. 
% with non-asymptotic confidence bound.   
%
%\textcolor{red}{TODO: add some discussion on upper - lower bound way in appendix \ref{sec:bilateral}} 
 %However, since in practice $f$ tends to be a complicated functionand the landscape of optimizing $\d$ is non-convex, there is no guarantee that $\min_{\d\in\mathcal{\B}}\pi_{\boldsymbol{0}}[f_{c}](\x+\d)$ can be exactly computed. 
%A natural idea is to find a lower bound of $\min_{\d\in\mathcal{\B}}\pi_{\boldsymbol{0}}[f_{c}](\x+\d)$  that can be exactly computed. 
%To address the problem, we consider 

We propose a \textbf{constrained adversarial certification (CAC)} framework, which yields a guaranteed lower bound for Eq.\ref{equ:p02}.
% based on constrained \functional optimization. 
The main idea is simple: 
assume $\mathcal F$ is a function class which is known to include $\ftrue$,   
then the following optimization immediately yields a guaranteed lower bound
\begin{equation}
\begin{split}
 \min_{\d\in\mathcal{\B}}\fpizero(\x+\d) 
\geq \min_{f\in\F} \min_{\d\in\mathcal{\B}}
\bigg \{  f_{\pizero}(\x + \d)
~~\mathrm{s.t.}~~
f_{\pizero}(\x) = \fpizero(\x) 
\bigg\}
,
\label{eq:constrant_opt}
\end{split}
\end{equation}

where we define $f_{\pizero}(\x) = \E_{\z\sim \pizero}[f(\x+\z)]$ for any given $f$. Then we need to search for the minimum value of $f_{\pizero}(\x+\d)$ %\E_{\pizero}[f(\x+\d)]:= \E_{\z\sim\pizero}[f(\x+\d+\z)]$
for all classifiers in $\mathcal F$ that satisfies $f_{\pizero}(\x) = \fpizero(\x)$. This obviously yields a lower bound once $\ftrue \in \mathcal F$. 
%{\color{red}$\ftrue$ vs. $f$ need to be distinguished}
%where $\ftrue$ is the base classifier (\emph{e.g.} a deep neural network)
%and $f$ is an arbitrary function in the function class $\mathscr {F}$.
%Obviously, 
If $\mathcal {F}$ includes only $\ftrue$,  
then the bound is exact, but is computationally prohibitive due to the difficulty of optimizing $\boldsymbol{\delta}$. 
%When 
The idea is then to choose $\F$ properly 
to incorporate rich information of $\ftrue$, while allowing us to 
calculate the lower bound in Eq.\ref{eq:constrant_opt} computationally tractably.
%still yield  computationally tractable lower bound. 
%be properly in order to yield bounds that are both tight and 
%where $\mathcal{F}$ can be viewed as the function class of $f_{c}$and $\mathcal{H}[\cdot]$ is some \functional that extract certain information ofthe \emph{given} classifier $f^*_{c}$ such that $f$ also holds the information.Notice that if $\mathcal{H}[\cdot]$ provides all information of $f_{c}$, thenthe solution of problem (\ref{eq:constrant_opt}) is exactly the solutionof problem (\ref{eq:certify}). $\mathcal{H}$, $\F$ and $\pi_{\boldsymbol{0}}$need to be properly chosen such that the problem \ref{eq:constrant_opt}is solvable. 
In this paper, we consider the set of all functions bounded in $[0,1]$, namely 
\begin{align}
\F_{[0,1]} & =\bigg \{f:f(\z)\in[0,1], \forall \z\in \mathbb{R}^d \bigg \}, 
\label{equ:f01}
\end{align}
which guarantees to include all $\ftrue$ by definition.
% There are other $\mathcal F$ that also yields computationally tractable bounds, including the $L_p$ functional space $\mathcal F =\{f \colon \norm{f}_{L_p} \leq v\}$, which we leave for future work.

Denote by $\V(\F, \B)$ the lower bound in Eq.\ref{eq:constrant_opt}. We can rewrite it into the following minimax form using the Lagrangian function, % of Eq.\ref{eq:constrant_opt}, we have 
\begin{equation}
\begin{split}
  \V(\F,\B) = \min_{f\in\F} \min_{\d\in\B} \max_{\lambda\in\R} L(f, \d, \lambda) \triangleq \min_{f\in\F} \min_{\d\in\B} \max_{\lambda\in\R}\bigg\{ f_{\pizero}(\x+\d) -\lambda(f_{\pizero}(\x) - \fpizero(\x)) \bigg\},    
\end{split}
\end{equation}
where $\lambda$ is the Lagrangian multiplier. 
Exchanging the $\min$ and $\max$ yields the following dual form. 
%This yields the following dual form of the bound. % which allows tractable computation. 
%{\color{red}Despite the problem of Equation(\ref{eq:constrant_opt}) is a constraint optimization problem on \functional space and vector space, we show that it can be exactly solved by using idea of Lagrange dual method.}

\begin{theorem} \label{thm:dual}
I) %Denote by $E(\mathcal F, \B)$ the lower bound in Eq.\ref{eq:constrant_opt}, and 
\textbf{(Dual Form)} Denote by $\pi_{\d}$ the distribution of $\z +\d$  when $\z \sim \pizero$.  
Assume $\F$ and $\B$ are compact set. 
We have the following lower bound of $\V(\F, \B)$: 
%Let $L(f, \d, \lambda)$ be the Lagrangian function of the constraint optimization 
%We have the following lower bound by duality: 
%we have 
\begin{equation}
\begin{split}
\V(\F, \B) \ge 
\max_{\lambda\geq0}   \min_{f\in\F}\min_{\d\in\B}  
L(f, \d, \lambda) = \max_{\lambda\ge 0}\left\{\lambda \fpizero(\x)-\max_{\d\in\mathcal{\B}}\D_{\F}\left(\text{\ensuremath{\lambda}}\pi_{\0}~\Vert~ \pi_{\d} \right)\right\},
\label{equ:mainlowerbound}
\end{split}
\end{equation}
where we define the discrepancy term $\D_{\F}\left(\text{\ensuremath{\lambda}}\pi_{\0} ~\Vert~ \pi_{\d} \right)$ as
$$
\max_{f \in \F} \Big \{ \lambda\E_{\z\sim\pizero} [f(\x+\z)] - \E_{\z\sim\pi_{\d}}[f(\x+\z)] \Big\}, 
$$
which measures the difference of $\lambda \pizero$ and $\pi_{\d}$ by seeking the maximum discrepancy of the expectation for $f\in \F$. 
As we will show later, the bound in (\ref{equ:mainlowerbound}) is computationally tractable with proper  $(\F, \B, \pizero)$.
%akin to integral probability metric (IPM). 
%In addition, when $\F$ and $\B$ are compact set,the strong duality holds and the inequality in Eq.\ref{equ:mainlowerbound} achieves equality. 

II) When $\F = \F_{[0,1]} :=\{f\colon f(x)\in [0,1], ~~x\in \R^d\}$, we have in particular 
\begin{align*}
\D_{\F_{[0,1]}}\left(\text{\ensuremath{\lambda}}\pi_{\0} ~\Vert~ \pi_{\d} \right) =  \int\left(\text{\ensuremath{\lambda}}\pi_{\0}(\z)-\pi_{\d}(\z)\right)_{+}d\z,
% = \E_{\z\sim \pizero} \left [ \left (\lambda - \frac{\pi_{\d}(\z)}{\pizero(\z)}\right)_+\right ]  
\end{align*}
where $(t)_+ = \max(0, t)$.
Furthermore, we have 
$0\leq \D_{\F_{[0,1]}}\left(\text{\ensuremath{\lambda}}\pi_{\0} ~\Vert~ \pi_{\d} \right) \leq \lambda$ for any $\pizero$, $\pi_{\d}$ and $\lambda >0$. 
Note that $\D_{\F_{[0,1]}}\left(\text{\ensuremath{\lambda}}\pi_{\0} ~\Vert~ \pi_{\d} \right)$ coincides with the total variation distance between $\pizero$ and $\pi_{\d}$ when $\lambda = 1$. 

III) (Strong duality) Suppose $\mathcal{F}=\mathcal{F}_{[0,1]}$ and suppose that for any $\lambda\ge0$, $\min_{\d\in\B}\min_{f\in\mathcal{F}_{[0,1]}}L\left(f,\d,\lambda\right)=\min_{f\in\mathcal{F}_{[0,1]}}L\left(f,\d^{*},\lambda\right)$, for some $\d^{*}\in\B$, we have 
\[
\V\left(\mathcal{F},\B\right)=\max_{\lambda\ge0}\min_{\d\in\B}\min_{f\in\mathcal{F}}L\left(f,\d,\lambda\right).
\]
\end{theorem}

\paragraph{Remark}
We will show later that the proposed methods and the cases we study satisfy the condition in part III of the theorem and thus all the lower bounds of the proposed method are tight.

% \begin{proof} 
% \end{proof}

Proof is deferred to Appendix~\ref{sec:proof_thm1}. Although the lower bound in  Eq.\ref{equ:mainlowerbound} still involves an optimization on $\vv\delta$ and $\lambda$, 
both of them are much easier than the original adversarial optimization in Eq.\ref{equ:p02}.
With proper choices of $\F$, $\B$ and $\pi_{\0}$, the optimization of $\vv\delta$ can be shown to provide simple closed-form solutions by exploiting the symmetry of $\mathcal{B}$, 
and the optimization of $\lambda$ is a very simple one-dimensional searching problem.

As corollaries of Theorem \ref{thm:dual}, we can exactly recover the bound derived by \cite{teng2020ell} and \cite{cohen2019certified} under our functional optimization framework, different from their original Neyman-Pearson lemma approaches.

\begin{corollary}
\label{thm:jiaye}
With Laplacian noise $\pizero(\cdot) = \laplace(\cdot; b)$, where $\laplace(\boldsymbol{x}; b) = \frac{1}{(2b)^d}\exp(-\frac{\Vert\boldsymbol{x}\Vert_1}{b})$, $\ell_1$ adversarial setting $\B = \{\d\colon \norm{\d}_{1}\leq r\}$ and $\F =\F_{[0,1]}$, the lower bound in Eq.\ref{equ:mainlowerbound} becomes
\begin{equation}
\begin{split}
\label{equ:resultjiaye}
 \max_{\lambda \ge 0} \left\{ \lambda \fpizero(\x) -\max_{\|\d\|_1\le r} \D_{\F_{[0,1]}}\left(\text{\ensuremath{\lambda}}\pi_{\0} \Vert \pi_{\d}\right) \right\}= \left\{  \begin{array}{cc}
             1 -  e^{r/b}(1 - \fpizero(\x)), \text{when} \fpizero(\x) \geq 1 - \frac{1}{2}e^{-r/b},\\
             &\\
             \frac{1}{2}e^{-\frac{r}{b}-\log[2(1-\fpizero(\x)]}, \text{when} \fpizero(\x) < 1 - \frac{1}{2}e^{-r/b}.
             \end{array}
\right.
\end{split}
\end{equation}
\end{corollary}

% \vspace{-4mm}
Thus, with our previous explanation, we obtain
$\V(\F,\B)\geq\frac{1}{2} \Longleftrightarrow r \le -b\log\left[2(1 - \fpizero(\x))\right]$,
which is exactly the $\ell_1$ certification radius derived by \cite{teng2020ell}. See Appendix~\ref{sec: retrievejiaye} for proof details. For Gaussian noise setting which has been frequently adopted, we have
\begin{corollary}\label{thm:cohen}
With isotropic Gaussian noise $\pizero = \normal(\vv 0, \sigma^2 I_{d\times d})$, 
$\ell_2$ attack $\B = \{\d \colon ~  \norm{\d}_{2}\leq r\}$ and $\F =\F_{[0,1]}$, the lower bound in Eq.\ref{equ:mainlowerbound} becomes
%Solving the optimization
\begin{equation}
\begin{split}
\label{equ:resultCohen}
 \max_{\lambda \ge 0} \left\{ \lambda \fpizero(\x) -\max_{\|\d\|_2\le r} \D_{\F_{[0,1]}}\left(\text{\ensuremath{\lambda}}\pi_{\0} \Vert \pi_{\d}\right) \right\} = \Phi\left(\Phi^{-1}(\fpizero(\x))  - \frac{r}{\sigma}\right).
\end{split}
\end{equation}
\end{corollary} 

Analogously, we can retrieve the main theoretical result of \cite{cohen2019certified} :$
\V(\F,\B)\geq\frac{1}{2} \Longleftrightarrow r \le \sigma \Phi^{-1}(\fpizero(\x)).$ See Appendix~\ref{sec: retrievecohen} for proof details.

% \begin{wrapfigure}{r}{0.2\textwidth}
% % \vspace{-15pt}
% %\includegraphics[width=.25\textwidth]{fig/l2_example.pdf} 
% \includegraphics[width=.175\textwidth]{fig/adv.pdf} 
% % \vspace{-10pt}
% % \vspace{-20pt}
% \label{fig:l2_example}
% \end{wrapfigure}

% One key step of the proof is to show that $\max_{\|\d\|_p\le r} \D_{\F_{[0,1]}} \left(\text{\ensuremath{\lambda}}\pi_{\0} \Vert \pi_{\d}\right)$ $p=1, 2$ is achieved when $\d$ is on the boundary of the $\ell_p$ ball $\mathcal B=\{\norm{\d}_p \leq r\}$, which can be, for example, $\d = [r,0,\ldots, 0]^\top$, due to symmetry of $\B$ (see figure on the right for $\ell_2$ case).  \red{zhe duan ting duo yu de}

\subsection{Trade-off Between Accuracy and Robustness}
The lower bound in Eq.\ref{equ:mainlowerbound} reflects an intuitive  trade-off between the robustness and accuracy on the certification problem:
\begin{equation}
\label{eq:objective2}
\max_{\lambda \geq 0}\Bigg [ \lambda\underset{\mathrm{Accuracy}}{\underbrace{ \fpizero(\x)}}+\underset{\mathrm{Robustness}}{\underbrace{
\left(-\max_{\d\in\B}\D_\F\left(\text{\ensuremath{\lambda}}\pi_{\0}~\Vert~ \pi_{\d} \right)\right)}} \Bigg ], 
\end{equation}
where the first term reflects the accuracy of the smoothed classifier (assuming the true label is $y=1$), 
while the second term $-\max_{\d\in\B}\D_\F\left(\text{\ensuremath{\lambda}}\pi_{\0}~\Vert~ \pi_{\d} \right)$  measures the robustness of the smoothing method, via the negative maximum discrepancy between the original smoothing distribution $\pi_{\0}$ and perturbed distribution $\pi_{\d}$ for $\d\in \B$. The maximization of dual coefficient $\lambda$ can be viewed as searching for a best balance between these two terms to achieve the largest lower bound. 
%and hence reflect the robustness of the model; 
%Once the distance is smaller,the model can be more robust.
%on the other hand, the first term $ \lambda \fpizero(\x)}$ represents the accuracy.  %assuming the original . %Once you choose a distribution which is not concentrated to the center, the accuracy will be as poor as possible.
%However, if a concentrated distribution is chosen, the distance between the two distributions should be large.

More critically, different choices of smoothing distributions  yields a trade-off between accuracy and robustness in Eq.\ref{eq:objective2}. 
% If $\pizero$ has large variance or large kurtosis, the distance $\D_\F(\lambda \pizero~\Vert~\pi_{\d})$ will tend to be small thus the smoothed model is robust. 
% However, if the variance of $\pizero$ is too large, we may obtain a low value of $\fpizero(\xt)$ and hence a less accurate model. 
A good choice of the smoothing distribution should \textcircled{\small{1}} be centripetal enough to obtain a large $\fpizero(\x)$ and \textcircled{\small{2}} have large kurtosis or long tail to yield a small $\max_{\d\in\B}\D_\F(\lambda \pizero~\Vert~\pi_{\d})$ discrepancy term. In the next section, we'll show how to design a distribution that could improve both points.

% well balance the accuracy and robustness, by distributing its mass properly to yield small $\max_{\d\in\B}\D_\F(\lambda \pizero~\Vert~\pi_{\d})$ and large $\fpizero(\x)$ simultaneously. 

\section{Improving Certification Bounds with a New Distribution Family}%Centripetal distribution: a new distribution family for better certification}
\label{sec:filling}

% In this section, we identify a key problem of the usage of Laplacian and Gaussian noise in high dimensional space, due to the ``thin shell'' phenomenon that the probability mass of them  concentrates on a sphere far away from the center points~\cite{vershynin2018high}. Motivated by this pivotal observation, we propose a new family of smoothing distributions that alleviate this problem for $\ell_1$, $\ell_2$ and $\ell_\infty$ attack.

\subsection{\textit{``Thin Shell''} Phenomenon and New Distribution Family} 
\label{sec:L2}

We first identify a key problem of the usage of Laplacian and Gaussian noise in high dimensional space, due to the ``\emph{thin shell}'' phenomenon that the probability mass of them  concentrates on a sphere far away from the center points~\cite{vershynin2018high}. 
%
% Although the isotropic Gaussian distribution appears to be a natural choice of the smoothing distribution, it's sub-optimal for the trade-off between accuracy and robustness in Eq.\ref{eq:objective2}, especially in high dimensions. 
% The key problem is that, in high dimensional spaces, the probability mass of Gaussian distributions concentrates on a% \emph{thin shell} away from the center:
%
\begin{proposition}
[\cite{vershynin2018high}, Section~3.1]
Let $\z \sim  \normal(\vec 0, ~ I_{d\times d})$ be a $d$-dimensional standard Gaussian random variable. Then there exists a constant $c$, such that for any $\delta \in (0,1)$,  
$\mathrm{Prob}$ $\left ( \sqrt{d} - \sqrt{c\log(2/\delta)} \leq \norm{\z}_2 \leq \sqrt{d} +  \sqrt{c\log(2/\delta)} \right) \geq 1-\delta.$
%$$\mathrm{Prob}( \sqrt{d-1} - c \leq \norm{z}_2 \leq \sqrt{d-1} + c) \geq 1-4\exp(-c^2/4)/c^2. $$
See \cite{vershynin2018high} for more discussion. 
\end{proposition}

This suggests that with high probability,
$\z$ takes values very close to the sphere of radius $\sqrt{d}$, within a constant distance from that sphere. There exists similar phenomenon for Laplacian distribution:
\begin{proposition}[Chebyshev bound]
Let $\z$ be a $d$-dimensional Laplacian random variable, $\z=(z_1, \cdots, z_d)$, where $z_i \sim \laplace(1), i = 1,\cdots, d$. Then for any $\delta \in (0,1)$, we have $\mathrm{Prob}$ $\left ( 1 - 1/\sqrt{d\delta} \leq \norm{\z}_1/d \leq 1 + 1/\sqrt{d\delta} \right) \geq 1-\delta.$
\end{proposition}

Although choosing isotropic Laplacian and Gaussian distribution appears to be natural, this ``\emph{thin shell}'' phenomenon makes it sub-optimal to use them for adversarial certification, because one would expect that the smoothing distribution should concentrate around the center (the original image) in order to make the smoothed classifier accurate enough in trade-off of Eq.\ref{eq:objective2}. 

% this paragraph was to be annotated
% To illustrate the problem, consider a simple example when the true classifier is $\ftrue(\vv x) = \mathbb{I}(\norm{\vv x - \x}_2\leq \epsilon\sqrt{d})$ for a constant $\epsilon <1$, where $\mathbb{I}$ is the indicator function. Then when the dimension $d$ is large, we would have $\ftrue(\x) = 1$ while $\fpizero(\x) \approx 0$ when $\pizero = \normal(\vec 0, I_{d\times d})$. It is of course possible to decrease the variance of $\pizero$ to improve the accuracy of the smoothed classifier $\fpizero$. However, this would significantly improve the distance term in Eq.\ref{eq:objective2} and does not yield an optimal trade-off on accuracy and robustness.

Thus it's desirable to design a distribution more \textit{concentrated} to center. To motivate our new distribution family, it's useful to examine the density function of the distributions of the radius of spherical distributions in general. 
\begin{proposition} 
 Assume $\z$ is a symmetric random variable on $\R^d$ with a probability density function (PDF) of form $\pizero(\z)\propto \phi(\norm{\z})$, where $\phi \colon [0,\infty) \to [0,\infty)$ is a univariate function, then the PDF of the norm of $\z$ is $p_{\norm{\z}}(r) \propto r^{d-1} \phi(r)$. The term $r^{d-1}$ arises due to the integration on the surface of radius $r$ norm ball in $\RR^d$. Here $\norm{\cdot}$ can be any $L_p$ norm.
\end{proposition}

In particular, for $\z \sim \normal(0, \sigma^2 I_{d\times d})$, 
we have $\phi(r) \propto \exp(-r^2/(2\sigma^2))$ and hence 
$p_{\norm{\z}_2}(r) \propto r^{d-1}\exp(-r^2/(2\sigma^2))$.
% , which is a scaled Chi distribution. 
We can see that the ``\emph{thin shell}'' phenomenon is caused by the  $r^{d-1}$ term, which makes the density to be highly peaked when $d$ is large. 
To alleviate the concentration phenomenon, we need to cancel out the effect of $r^{d-1}$, which motivates the following family of smoothing distributions: 
$$
\pi_{\boldsymbol{0}}(\z)
\propto\left\Vert \z\right\Vert _{n_1}^{-k}\exp\left(-\frac{\left\Vert \z\right\Vert _{n_2}^{p}}{b}\right),
$$
where parameters $k, n_1, n_2, p \in \mathbb{N}$. Next we discuss how to choose suitable parameters depending on specific perturbation region.

% \vspace{40pt}
\subsection{$\ell_1$ and $\ell_2$ Region Certification} 
\label{sec:L2}

% \vspace{-10pt}
\begin{wrapfigure}{r}{0.4\textwidth}  
\vspace{0pt}
\begin{tabular}{c}
     \raisebox{3.5em}{\rotatebox{90}{\small Variance}} 
     \includegraphics[width=0.35\textwidth]{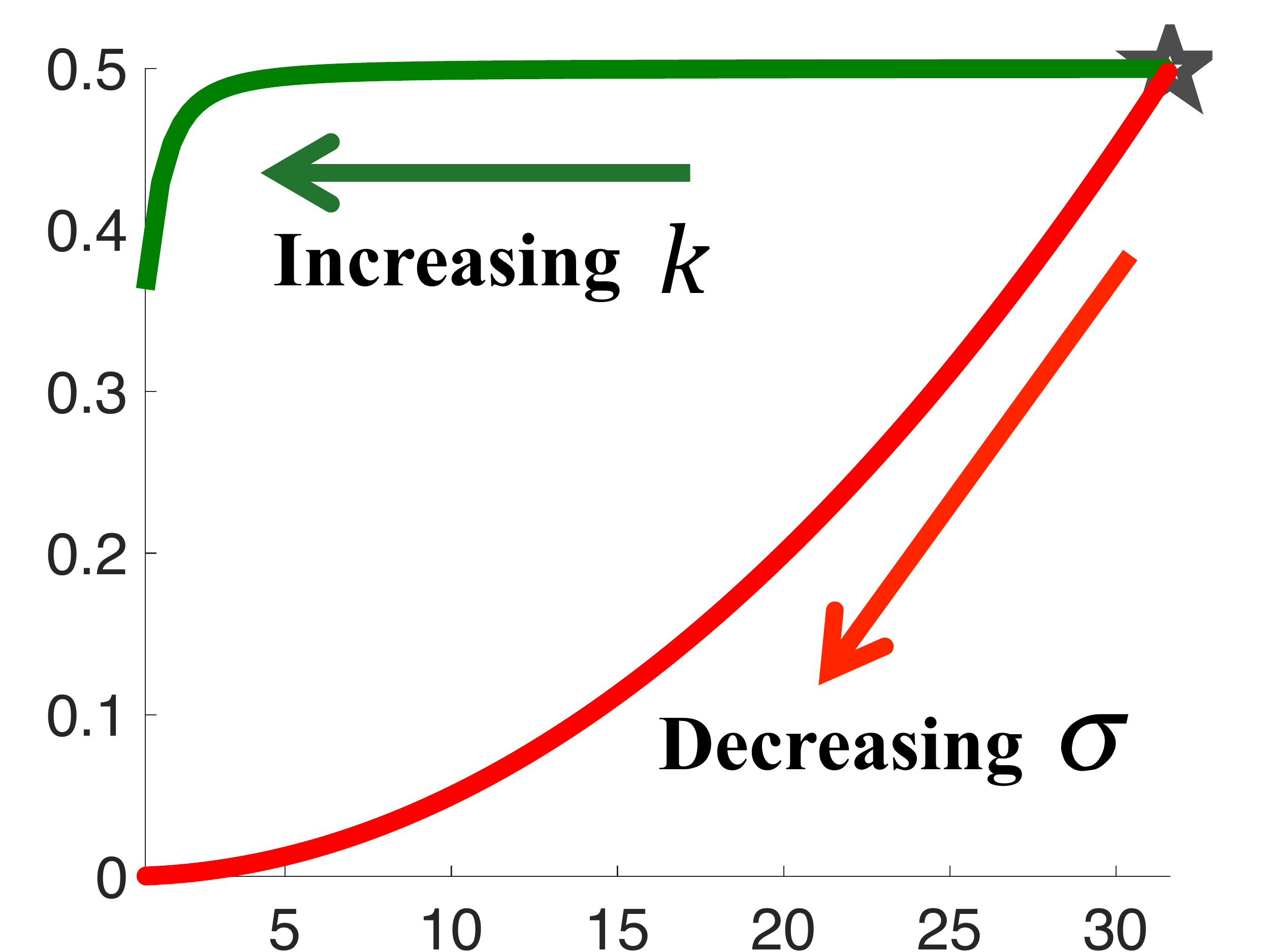} \\
     {\small Mean}\\
\end{tabular}
% \vspace{-10pt}
\caption{\small Starting from radius distribution in Eq.\ref{equ:L2smooth} with $d=100$ $\sigma=1$ and $k=0$ (black start),  
increasing $k$ (green curve) moves the mean towards zero \emph{without
significantly reducing the variance}. 
Decreasing $\sigma$ (red curve) can also decrease the mean, but with a cost of decreasing the variance quadratically.  
}
\vspace{-30pt}
\label{fig:meanvariance}
\end{wrapfigure}

Based on original Laplacian and Gaussian distributions and above intuition, we propose:
% \vspace{-2pt}
\begin{equation}
    \begin{split}
        \text{}\ell_1: \pi_{\boldsymbol{0}}(\z)
        \propto\left\Vert \z\right\Vert _{1}^{-k}\exp\left(-\frac{\left\Vert \z\right\Vert _{1}^{}}{b}\right)   
        % \quad\quad \text{}\ell_2: \pi_{\boldsymbol{0}}(\z) \propto\left\Vert \z\right\Vert _{2}^{-k}\exp\left(-\frac{\left\Vert \z\right\Vert _{2}^{2}}{2\sigma^{2}}\right), 
        % \text{and hence} \quad 
        % p_{\norm{\z}_2}(r)\propto r^{d-k-1} \exp\left (-\frac{r^2}{2\sigma^2} \right), 
    \label{equ:L1smooth}
    \end{split}
\end{equation} 

\vspace{-12pt}
\begin{equation}
    \begin{split}
        \ell_2: \pi_{\boldsymbol{0}}(\z) \propto\left\Vert \z\right\Vert _{2}^{-k}\exp\left(-\frac{\left\Vert \z\right\Vert _{2}^{2}}{2\sigma^{2}}\right)
    \label{equ:L2smooth}
    \end{split}
\end{equation} 

% \begin{equation}
%     \begin{split}
%         \pi_{\boldsymbol{0}}(\z)
%         \propto\left\Vert \z\right\Vert _{1}^{-k}\exp\left(-\frac{\left\Vert \z\right\Vert _{1}^{}}{b}\right), 
%          \text{and hence}  \quad 
%         p_{\norm{\z}_1}(r)\propto r^{d-k-1} \exp\left (-\frac{r}{b} \right).
%     \label{equ:L1smooth}
%     \end{split}
% \end{equation} 
where we introduce the $\left\Vert \z\right\Vert^{-k}$ term in $\pizero$, with $k$ a positive parameter, to make the radius distribution more concentrated when $k$ is large. 
% As for $\ell_1$ certification, we also introduce a $\norm{z}_1^{-k}$ term to force a more centralized distribution:

The radius distribution in Eq.\ref{equ:L1smooth}
and Eq.\ref{equ:L2smooth} 
is controlled by two parameters: $\sigma$ (or $b$) and $k$, who control the scale and shape of the distribution,  respectively. 
The key idea is that adjusting extra parameter $k$ allows us to control the trade-off the accuracy and robustness more precisely. 
As shown in Fig.\ref{fig:meanvariance}, adjusting $\sigma$ moves the mean close to zero (hence \textcircled{\small{1}} yielding higher accuracy), but at cost of decreasing the variance quadratically (hence \textcircled{\small{2}} less robust). In contrast, adjusting $k$ decreases the mean without significantly impacting the variance, thus yield a much better trade-off on accuracy and robustness.

% \begin{figure} % {r}{0.35\textwidth}  
% % \vspace{-10pt}
% \begin{tabular}{c}
%      \raisebox{4.5em}{\rotatebox{90}{\large Variance}} 
%      \includegraphics[width=0.4\textwidth]{fig/gamma_mean_variance_modify.pdf} \\
%      {\large ~~Mean}\\
% \end{tabular}
% % \vspace{-10pt}
% \caption{\small Starting from radius distribution in Eq.\ref{equ:L2smooth} with $d=100$, $\sigma=1$ and $k=0$ (black star),  
% increasing $k$ (green curve) allows us to move the mean towards zero \emph{without
% significantly reducing the variance}. 
% Decreasing $\sigma$ (red curve) can also decrease the mean, but with a cost of decreasing the variance quadratically.}
% % \vspace{-10pt}
% \label{fig:meanvariance}
% \end{figure}

\vspace{-5pt}
\paragraph{Computational Method}
Now we no longer have the closed-form solution of the bound like Eq.\ref{equ:resultjiaye} and Eq.\ref{equ:resultCohen}. However, efficient computational methods can still be developed for calculating the bound in  Eq.\ref{equ:mainlowerbound} with $\pizero$ in Eq.\ref{equ:L2smooth} or Eq.\ref{equ:L2smooth}. The key is that the maximum of the distance term 
 $\D_{\F_{[0,1]}}(\lambda \pizero~||~ \pi_{\d})$ over $\d\in \B$ 
 is always achieved on the boundary of $\B$: 

\begin{theorem} \label{thm:opt_mu_l2}
Consider the $\ell_{1}$ attack with $\boldsymbol{\mathcal{B}}=\left\{ \d:\left\Vert \d\right\Vert  _{1}\le r\right\}$ and smoothing distribution  $\pi_{\0}(\z)\propto\left\Vert \z\right\Vert _{1}^{-k}\exp\left(-\frac{\left\Vert \z\right\Vert _{1}^{}}{b}\right)$ with $k\ge0$ and $b>0$, or the $\ell_{2}$ attack with $\boldsymbol{\mathcal{B}}=\left\{ \d:\left\Vert \d\right\Vert  _{2}\le r\right\}$ and smoothing distribution  $\pi_{\0}(\z)\propto\left\Vert \z\right\Vert _{2}^{-k}\exp\left(-\frac{\left\Vert \z\right\Vert _{2}^{2}}{2\sigma^{2}}\right)$ with $k\ge0$ and $\sigma>0$.  
Define $\d^{*} = [r,0,...,0]^\top$, we have 
$$
\mathbb{D}_{\mathcal{F}_{[0,1]}}\left(\lambda\pi_{\0}~\Vert~\pi_{\d^*}\right) = \max_{\delta\in\B}\mathbb{D}_{\mathcal{F}_{[0,1]}}\left(\lambda\pi_{\0}~\Vert~\pi_{\d}\right) 
$$
for any positive $\lambda$.
\end{theorem}

% See proof in Appendix. 
With Theorem~\ref{thm:opt_mu_l2}, we can compute Eq.\ref{equ:mainlowerbound} with $\d = \d^*$. We then calculate $\mathbb{D}_{\mathcal{F}_{[0,1]}}\left(\lambda\pi_{\0}~\Vert~\pi_{\d^*}\right)= \E_{\z\sim \pizero}\left [\left (\lambda-\frac{\pi_{\d^*}(\z)}{\pizero(\z)}\right )_+\right ]$ using Monte Carlo approximation with i.i.d. samples $\{\z_{i}\}_{i=1}^n$ be i.i.d. samples from $\pizero$: $\hat D := \frac{1}{n}\sum_{i=1}^n \left (\lambda-{\pi_{\d^*}(\z_i)}/{\pizero(\z_i)}\right )_+,$ which is bounded in the following confidence interval $[\hat D - \lambda\sqrt{\log(2/\delta)/({2n})}, \hat D+ \lambda\sqrt{\log(2/\delta)/({2n})}]$ with confidence level $1-\delta$ for $\delta \in (0,1)$.
% Drawing a sufficiently large number of samples could achieve approximation with arbitrary accuracy. 
What's more, the optimization on $\lambda\geq 0$ is one-dimensional and can be solved numerically efficiently (see Appendix for details).

\subsection{$\ell_\infty$ Region Certification}
\label{sec:Linfty} 

Going further, we consider the more difficult $\ell_\infty$ attack whose attacking region is $\B_{\ell_\infty, r}=\{\d \colon \norm{\d}_\infty \leq r\}$. 
The commonly used Gaussian smoothing distribution, as well as our $\ell_2$-based smoothing distribution in Eq.\ref{equ:L2smooth}, is unsuitable for this region:

\begin{proposition}
\label{cor: transformation}
With the smoothing distribution $\pizero$ in Eq.\ref{equ:L2smooth}
for $ k\ge0, \sigma>0$,  and $\F = \F_{[0,1]}$ shown in Eq.\ref{equ:f01}, 
the bound we get for  
certifying the $\ell_{\infty}$ attack on  $\B_{\ell_\infty, r}=\{\d:\left\Vert \d\right\Vert _{\infty}\le r\}$ is equivalent 
to that for certifying the $\ell_{2}$ attack on $\B_{
\ell_2, \sqrt{d}r}=\{\d:\left\Vert  \d\right\Vert _{2}\le {\sqrt{d}}r\}$, that is, 
$$
\V(\F_{[0,1]}, ~ \B_{\ell_\infty, r}) = 
\V(\F_{[0,1]}, ~ \B_{\ell_2, \sqrt{d} r}). 
$$
\end{proposition}

\begin{wrapfigure}{r}{0.45\textwidth}  
\vspace{-10pt}
% \begin{tabular}{c}
%      \raisebox{3.5em}{\rotatebox{90}{\small Variance}} 
%      \includegraphics[width=0.3\textwidth]{fig/gamma_mean_variance_modify.pdf} \\
%      {\small Mean}\\
% \end{tabular}
\vspace{5pt}
\includegraphics[width=.45\textwidth]{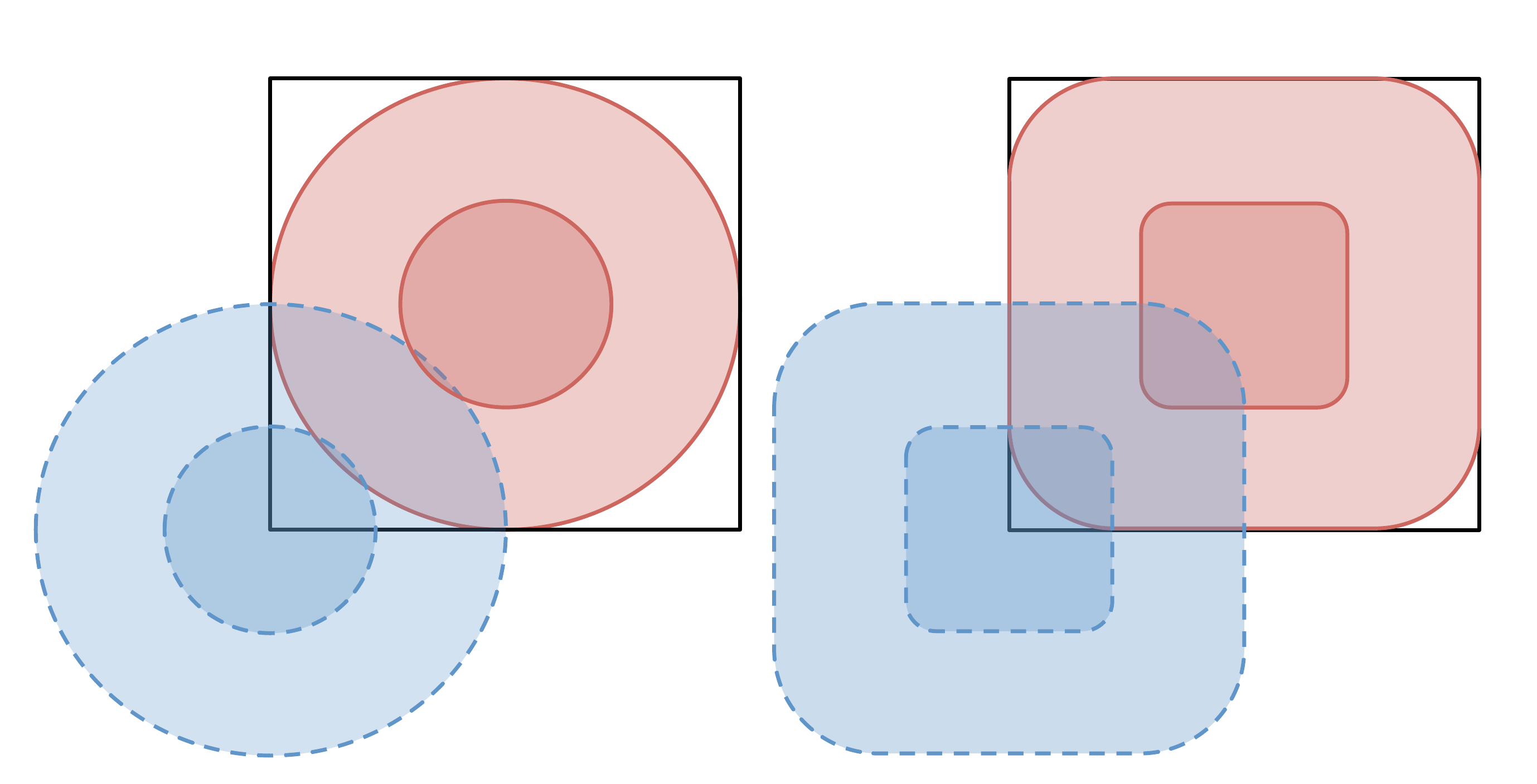}
\vspace{-5pt}
\caption{
     For $\ell_\infty$ attacking, 
     compared with the distribution in Eq.\ref{equ:L2smooth},
     the mixed norm distribution in ~Eq.\ref{equ:mixednormfamily} (right) 
     yields smaller discrepancy term (because of larger overlap areas), and hence higher robustness and better confidence bound.
     The distribution described in ~Eq.\ref{eq:inf_inf_distribution} has the same impact.}
\vspace{-30pt}
\label{fig:c}
\end{wrapfigure}

As shown in this proposition, if we use $\ell_2$ distribution in Eq.\ref{equ:L2smooth} for certification, the bound we obtain is effectively the bound we get for verifying a $\ell_2$ ball with radius $\sqrt{d} r$, which is too large to give meaningful results due to high dimension.
% The maximum distance $\max_{\d \in \B_{\ell_\infty, r}} \D_\F(\lambda \pizero~\Vert~ \pi_{\d})$ is achieved at one of these pointy points, i.e., $\d^*=[\sqrt{d}r,0,...,0]^\top$, making it equivalent to optimizing in the $\ell_2$ ball with radius $\sqrt{d}r.$ 

In order to address this problem, 
we extend our proposed distribution family with new distributions which are more suitable for $\ell_\infty$ certification setting:
% We propose the following two new smoothing distributions for $\ell_\infty$ certification:
\begin{equation}
\label{eq:inf_inf_distribution}
\pi_{\boldsymbol{0}}(\z) \propto  \left\Vert \z\right\Vert _{\infty}^{-k}\exp\left(-\frac{\left\Vert \z\right\Vert _{\infty}^{2}}{2\sigma^{2}}\right),
\end{equation}

\begin{equation}
\label{equ:mixednormfamily}
\pizero(\z) \propto \norm{\z}_{\infty}^{-k} \exp\left (- \frac{\norm{\z}_2^2}{2\sigma^2} \right ).
\end{equation}

% Though it seems a natural choice to perform certification with Eq.\ref{eq:inf_inf_distribution}, this \emph{pure $\ell_\infty$ norm} distribution family doesn't work effectively for $\ell_\infty$ attacks, partially due to the extremely large volume of the $\ell_\infty$ ball (\emph{e.g.}, compared with the volume of $\ell_2$ ball). See Section~\ref{sec:l_inf_exp} for a numerical justification, and a formalized version of this claim is defered to Appendix. Hence in Eq.\ref{equ:mixednormfamily}, we propose to use a mix of $\ell_2$ and $\ell_\infty$ norms and replace the $\norm{\z}_2^{-k}$ term in Eq.\ref{equ:L2smooth} with $\norm{\z}_\infty^{-k}$. 
The motivation is to allocate more probability mass along the ``pointy'' directions with larger $\ell_\infty$ norm, and hence decrease the maximum discrepancy term $\max_{\d \in \B_{\ell_\infty, r}} \D_\F(\lambda \pizero~\Vert~ \pi_{\d})$, see  Fig.\ref{fig:c}.

\paragraph{Computational Method}
In order to compute the lower bound with proposed distribution, we need to establish similar theoretical results as Theorem~\ref{thm:opt_mu_l2}, showing the optimal $\d$ is achieved at one vertex (the pointy points) of $\ell_\infty$ ball. 

\begin{theorem} \label{thm:opt_mu_linf}
Consider the $\ell_{\infty}$ attack with  ${\B}_{\ell_\infty,r}=\left\{ \d:\left\Vert \d\right\Vert  _{\infty}\le r\right\}$ and the mixed norm smoothing distribution in Eq.\ref{equ:mixednormfamily}
%$\pi_0$  $\pi_{\0}(\z)\propto\left\Vert \z\right\Vert _{\infty}^{-k}\exp\left(-\frac{\left\Vert \z\right\Vert _{2}^{2}}{2\sigma^{2}}\right)$ 
with $k\ge0$ and $\sigma>0$. 
Define $\d^{*} = [r,r,...,r]^\top$. 
We have for any $\lambda > 0$, 
$$
\mathbb{D}_{\mathcal{F}_{[0,1]}}\left(\lambda\pi_{\0}~\Vert~\pi_{\d^*}\right) = \max_{\delta\in\B}\mathbb{D}_{\mathcal{F}_{[0,1]}}\left(\lambda\pi_{\0}~\Vert~\pi_{\d}\right). 
$$
%where . 
%Consider the $\ell_{\infty}$ attack with $\boldsymbol{\mathcal{B}}=\left\{ \d:\left\Vert \d\right\Vert _{\infty}\le r\right\}$, given $\pi_{\0}(\z)\propto\left\Vert \z\right\Vert _{\infty}^{-k}\exp\left(-\frac{\left\Vert \z\right\Vert _{2}^{2}}{2\sigma^{2}}\right)$ for some $k\ge0$ and $\sigma>0$,
%we obtain $$\d^{*}=[r,r,...,r]^{\top}=\underset{\d\in\boldsymbol{\mathcal{B}}}{\arg\max}\ \mathbb{D}_{\mathcal{F}_{[0,1]}}\left(\lambda\pi_{\0}\parallel\pi_{\d}\right)$$for any $\lambda\ge0$. 
%Given a $d$ dimension perturbation $\|\d\|_\infty \leq r$ and a distribution $\pi_{\boldsymbol{0}}(\z)&\propto\left\Vert \z\right\Vert _\infty^{-k}\exp\left(-\frac{\left\Vert \z\right\Vert _{2}^{2}}{2\sigma_r^{2}}\right)$,
%$\d^* = \arg\max_{\d\in\mathcal{\B}}\D_+\left(\text{\ensuremath{\lambda}}\pi_{\0}(\cdot)\Vert \pi_{\d} (\cdot)\right)$ if and only if $|\d^*_i| = r$ for any $i \in \{1, \dots, d\}$.
\end{theorem}
The proofs of Theorem \ref{thm:opt_mu_l2} and \ref{thm:opt_mu_linf} are non-trivial and deferred to Appendix.  With the optimal $\d^*$ found above, we can calculate the bound with similar Monte Carlo approximation outlined in Section~\ref{sec:L2}. 

\section{Experiments}
We evaluate proposed certification bound and smoothing distributions for $\ell_1$, $\ell_2$ and $\ell_\infty$ attacks. We compare with the randomized smoothing method of \cite{teng2020ell} with Laplacian smoothing for $\ell_1$ region cerification.
For $\ell_2$ and $\ell_\infty$ cases, 
we regard the method derived by \cite{cohen2019certified} with Gaussian smoothing distribution as the baseline. 
%For $\ell_2$ attacks, 
%we compare to the method based on isotropic Gaussian distribution in  \cite{cohen2019certified}. 
%For $\ell_\infty$ perturbation, we also use Gaussian distribution as baseline.
% For $\ell_1$ perturbation, 
% we choose Gaussian distribution and Laplace distribution as the baselines.
For fair comparisons, we use the same model architecture and pre-trained models provided by \cite{teng2020ell}, \cite{cohen2019certified} and \cite{salman2019provably}, 
which are ResNet-110 for CIFAR-10 and ResNet-50 for ImageNet. We use the official code\footnote{\url{https://github.com/locuslab/smoothing}. Our results are slightly different with those in original paper due to the randomness of sampling.} provided by ~\cite{cohen2019certified} for all the following experiments. For all other details and parameter settings, we refer the readers to Appendix~\ref{sec: hyperparameters}.

\begin{table}[tbhp]
% \begin{table}[t]
% \vspace{-3pt}
\begin{center}
% \begin{footnotesize}
\begin{sc} % text font
\setlength\tabcolsep{2.5pt}
		\begin{tabular}{l|ccccccccc}
			\toprule
            $\ell_1$ Radius (CIFAR-10) & $0.25$ & $0.5$ & $0.75$ & $1.0$ & $1.25$ & $1.5$ & $1.75$ & $2.0$ & $2.25$ \\
            \midrule
            Baseline (\%) & 62 & 49 & 38 & 30 & 23 & 19 & 17 & 14 & 12 \\
            Ours (\%) & \bf{64} & \bf{51} & \bf{41} & \bf{34} & \bf{27} & \bf{22} & \bf{18} & \bf{17} & \bf{14} \\
            \bottomrule
		\end{tabular}
		\vspace{10pt}
		
		\setlength\tabcolsep{5pt}
		\begin{tabular}{l|ccccccc}
	\toprule
$\ell_1$ Radius (ImageNet)& $0.5$ & $1.0$ & $1.5$  & $2.0$ & $2.5$ & $3.0$ & $3.5$ \\
\midrule
Baseline (\%) & 50 & 41 & 33 & 29 & 25 & 18 & 15  \\
Ours (\%) & \bf{51} & \bf{42} & \bf{36} & \bf{30} & \bf{26} & \bf{22} & \bf{16}\\
            \bottomrule
		\end{tabular}
	\end{sc}
    % \end{footnotesize}
	\end{center}
	\caption{\label{table:l1-CIFAR10} Certified top-1 accuracy of the best classifiers with various $\ell_1$ radius.
% 	on CIFAR-10 and ImageNet.
	}%, on the models in  \cite{cohen2019certified}.}
	%and do not train any new models. }
\end{table}

\vspace{-7pt}
\paragraph{Evaluation Metrics}
Methods are evaluated with the certified accuracy defined in \cite{cohen2019certified}. 
Given an input image $\boldsymbol{x}$ and a perturbation region $\B$,
the smoothed classifier  
 certifies image $\boldsymbol{x}$ correctly if the prediction is correct and has a guaranteed confidence lower bound larger than $1/2$ for any $\d\in \B$.
The certified accuracy is the percentage of images that are certified correctly. Following \cite{salman2019provably}, we calculate the certified accuracy of all the classifiers in  \cite{cohen2019certified} or \cite{salman2019provably} for various radius, and report the best results over all of classifiers.

% \vspace{-10pt}
% \begin{table}[htbp]
% 	\begin{center}
% 	\begin{footnotesize}
%     \begin{sc}  
%     \setlength\tabcolsep{5pt}
% 		\begin{tabular}{l|ccccccc}
% 	\toprule
% $\ell_1$ Radius & $0.5$ & $1.0$ & $1.5$  & $2.0$ & $2.5$ & $3.0$ & $3.5$ \\
% \midrule
% Baseline (\%) & 50 & 41 & 33 & 29 & 25 & 18 & 15  \\
% Ours (\%) & \bf{51} & \bf{42} & \bf{36} & \bf{30} & \bf{26} & \bf{22} & \bf{16}\\
%             \bottomrule
% 		\end{tabular}
% 		\end{sc}
%     \end{footnotesize}
% 	\end{center}
% % 	\vspace{-10pt}
% 	\caption{\label{table:l1-imagenet} Certified top-1 accuracy of the best classifiers with various $\ell_1$ radius on ImageNet.}
%  	\vspace{-10pt}
% \end{table}

\subsection{$\ell_1$ \& $\ell_2$  Certification}
For $\ell_1$ certification, we compare our method with \cite{teng2020ell} on CIFAR-10 and ImageNet  with the type 1 trained model in \cite{teng2020ell}. As shown in Table~\ref{table:l1-CIFAR10}, our non-Laplacian centripetal distribution consistently outperforms the result of baseline for any $\ell_1$ radius.
% The ImageNet result is in Table~\ref{table:l1-imagenet}, which shows that our method outperforms the baseline's result uniformly for all $\ell_1$ radius.

\begin{table}[thbp]
% \begin{table}[t]
% \vskip 0.15in
\begin{center}
% \begin{footnotesize}
\begin{sc}
\setlength\tabcolsep{2.5pt}
		\begin{tabular}{l|ccccccccc}
			\toprule
            $\ell_2$ Radius (CIFAR-10) & $0.25$ & $0.5$ & $0.75$ & $1.0$ & $1.25$ & $1.5$ & $1.75$ & $2.0$ & $2.25$ \\
            \midrule
            Baseline (\%) & 60 & 43 & 34 & 23 & 17 & 14 & 12 & 10 & 8 \\
            Ours (\%) & \bf{61} & \bf{46} & \bf{37} & \bf{25} & \bf{19} & \bf{16} & \bf{14} & \bf{11} & \bf{9} \\
            \bottomrule
		\end{tabular}
		\vspace{10pt}
		
		\setlength\tabcolsep{5pt}
		\begin{tabular}{l|ccccccc}
			\toprule
            $\ell_2$ Radius (ImageNet)& $0.5$ & $1.0$ & $1.5$ & $2.0$ & $2.5$ & $3.0$ & $3.5$ \\
            \midrule
            Baseline (\%) & 49 & 37 & 29 & 19 & 15 & 12 & 9\\
            Ours (\%) & \bf{50} & \bf{39} & \bf{31} & \bf{21} & \bf{17} & \bf{13} & \bf{10}\\
            \bottomrule
		\end{tabular}
	\end{sc}
    % \end{footnotesize}
	\end{center}
	\caption{\label{table:l2-CIFAR-10} Certified top-1 accuracy of the best classifiers with  various $\ell_2$ radius.}% on CIFAR-10.}%, on the models in  \cite{cohen2019certified}.}
	%and do not train any new models. }
\end{table}

% \begin{table}[thbp]
% 	\begin{center}
% 	\begin{footnotesize}
%     \begin{sc}  
%     \setlength\tabcolsep{5pt}
% 		\begin{tabular}{l|ccccccc}
% 			\toprule
%             $\ell_2$ Radius & $0.5$ & $1.0$ & $1.5$ & $2.0$ & $2.5$ & $3.0$ & $3.5$ \\
%             \midrule
%             Baseline (\%) & 49 & 37 & 29 & 19 & 15 & 12 & 9\\
%             Ours (\%) & \bf{50} & \bf{39} & \bf{31} & \bf{21} & \bf{17} & \bf{13} & \bf{10}\\
%             \bottomrule
% 		\end{tabular}
% 		\end{sc}
%     \end{footnotesize}
% 	\end{center}
% 	\caption{\label{table:l2-imagenet} Certified top-1 accuracy of the best classifiers with various $\ell_2$ radius on ImageNet.}
% 	%We use the same model as \cite{cohen2019certified} and do not train any new models. }
% \end{table}

% \subsection{$\ell_2$ Certification}
\vspace{-10pt}
For $\ell_2$ certification, we compare our method with \cite{cohen2019certified} on CIFAR-10 and ImageNet. 
%when the certification region is defined as a $\ell_2$ ball.
For a fair comparison, we use the same pre-trained models as \cite{cohen2019certified}, which is trained with Gaussian noise on both CIFAR-10 and ImageNet dataset.
Table~\ref{table:l2-CIFAR-10} 
% and Table~\ref{table:l2-imagenet}
reports the certified accuracy of our method 
% with the non-Gaussian smoothing distribution in \eqref{equ:L2smooth} 
and the baseline on CIFAR-10 and ImageNet 
% respectively
. We find that our method consistently outperforms the baseline.
% For example, it enhances the 34\% accuracy to 37\% for 0.75 radius on CIFAR-10. It also imrpove the top-1 accuracy from $37\%$ to $39\%$ for 1.0 radius on ImageNet.
% Since we fix the same $k$ across all the models and all the radius, the improvement cannot be obtained by tuning $\sigma^2$ fine-grainedly. 
The readers are referred  to the Appendix~\ref{sec:details} for detailed ablation studies.

% \vspace{-5pt}
\subsection{$\ell_\infty$ Certification}
\label{sec:l_inf_exp}
\paragraph{Toy Example}
We first construct a simple toy example to verify the advantages of the distribution Eq.\ref{equ:mixednormfamily}  and Eq.\ref{eq:inf_inf_distribution} over the $\ell_2$ family in Eq.\ref{equ:L2smooth}. We set the true classifier to be $\ftrue(\vv x) = \mathbb{I}(\norm{x}_2\leq r)$  in $r=0.65$, $d=5$ case and plot in Fig.\ref{fig:linf_example} the Pareto frontier of the accuracy and robustness terms in Eq.\ref{eq:objective2} for the three families of smoothing distributions, as we search for the best combinations of parameters $(k, \sigma)$. 
 The mixed norm smoothing distribution clearly obtain the best trade-off on accuracy and robustness, and hence guarantees a tighter lower bound for certification.  
Fig.\ref{fig:linf_example} also shows that Eq.\ref{eq:inf_inf_distribution} even performs worse than Eq.\ref{equ:L2smooth}.
We further theoretically show that Eq.\ref{eq:inf_inf_distribution} is provably not suitable for $\ell_\infty$ region certification in Appendix~\ref{sec:lin_impossible}.

% \begin{figure}[htbp]
\begin{figure*}[t]
\begin{minipage}{0.5\textwidth}
\begin{tabular}{c}
     \raisebox{2.0em}{\rotatebox{90}{Accuracy}}
     \includegraphics[width=0.9\textwidth]{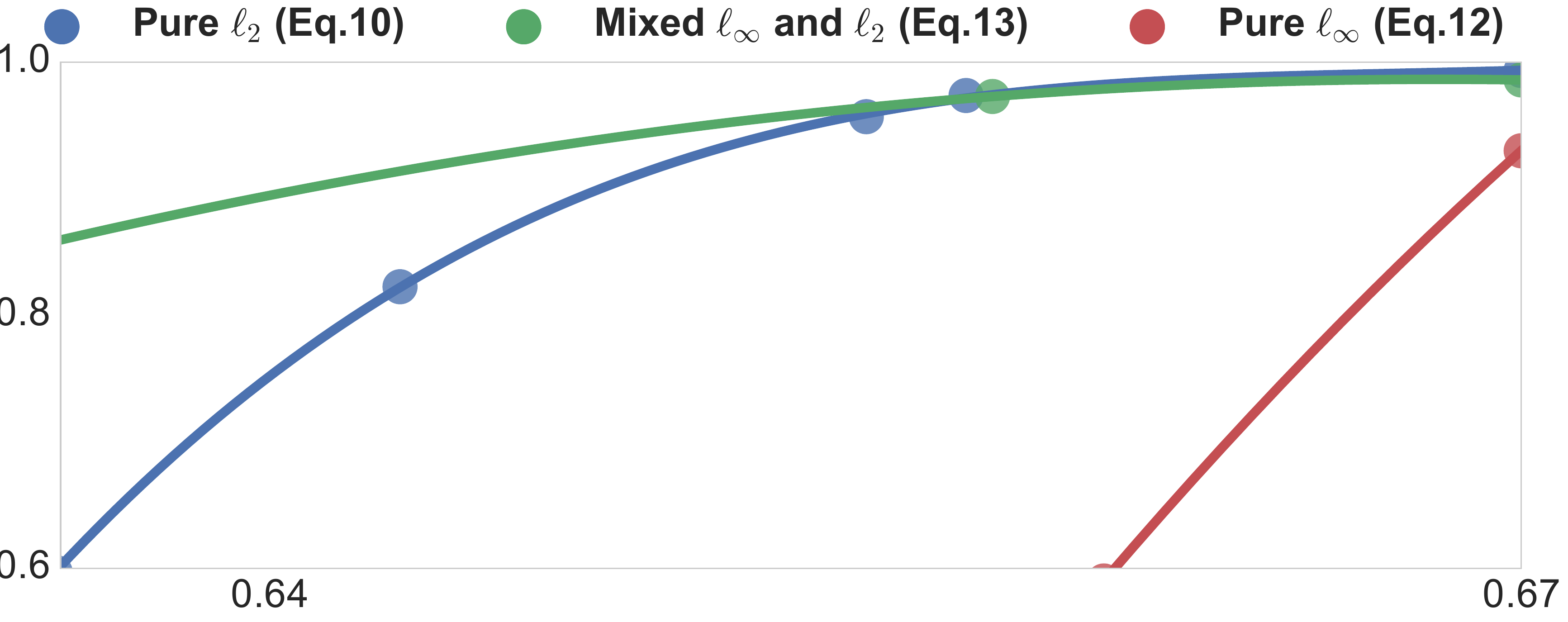} 
     \\
     { \centering Discrepancy Term (Robustness)}
    %  \\
\end{tabular}
\vspace{-5pt}
\caption{\small
The Pareto frontier of accuracy and robustness (in the sense of Eq.\ref{eq:objective2}) 
of the three smoothing families in Eq.\ref{equ:L2smooth}, Eq.\ref{equ:mixednormfamily}, and Eq.\ref{eq:inf_inf_distribution} for $\ell_{\infty}$ attacking,  
when we search for the best parameters $(k,\sigma)$ for each of them. The mixed norm family Eq.\ref{equ:mixednormfamily} yields the best trade-off than the other two. 
We assume $\ftrue(\vec x)=\mathbb{I}(\norm{\vec x}_2\leq r)$  and dimension $d=5$. The case when $\ftrue(\vec x)=\mathbb{I}(\norm{\vec x}_\infty\leq r)$ has similar result (not shown). 
}
\label{fig:linf_example}
\end{minipage}
\hfill
\begin{minipage}{0.45\textwidth}
% \begin{tabular}{c}
\small{~~~~~~$\sigma_0=1.00$~~~~~~~$\sigma_0=0.50$~~~~~~$\sigma_0=0.25$~~~~~~~~~~}\\
\raisebox{0.em}{\rotatebox{90}{\scriptsize Log [Acc (\%)] }} 
\includegraphics[width=1.\textwidth]{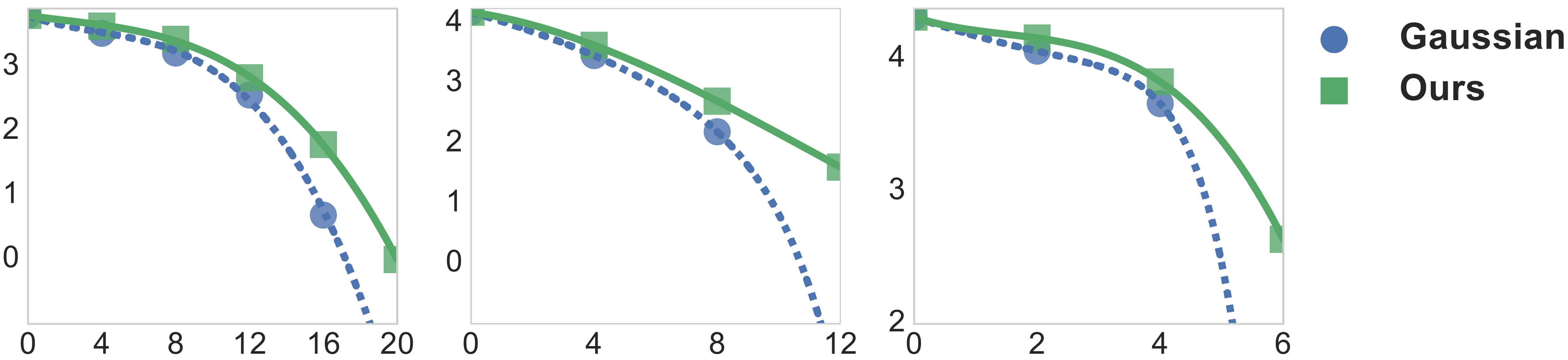} \\
\vspace{-1\baselineskip}
 \small ~~~~~~~~~~~~~~~~~~~~~~$\ell_\infty$ radius $\times$ 255~~~~~~~~~~~~\\
% \end{tabular}
\caption{
Results of $\ell_\infty$ verification on CIFAR-10, 
on models trained with Gaussian noise data augmentation with different variances $\sigma_0$. 
Our method obtains consistently better results.}
\vspace{-10pt}
\label{fig:linf}
\end{minipage}

\end{figure*}

\paragraph{CIFAR-10}
Based on above results, we only compared the method defined by Eq.\ref{equ:mixednormfamily} with \cite{salman2019provably} on CIFAR-10.
% using the models trained by \cite{salman2019provably}. 
The certified accuracy of our method and the baseline using Gaussian smoothing distribution and Proposition \ref{cor: transformation} are shown in Table~\ref{table:linf-CIFAR-10}. 
We can see that our method consistently outperforms the Gaussian baseline by a large margin. More clarification about $\ell_\infty$ experiments is in Appendix~\ref{sec:l_inf_clarification}. 
% which empirically shows our distribution is a more suitable distribution for $\ell_\infty$ perturbation.
% As shown in Figure~\ref{fig:linf_example}, 
% it can make the distance between $\pi_{\0}$ and $\pi_{\d}$, which indicates more robustness.

% \begin{figure}[htbp]
% \centering
% \begin{tabular}{c}
% \small{$\sigma_0=1.00$ ~~~~~~~~~~~~$\sigma_0=0.50$~~~~~~~~~~~~~$\sigma_0=0.25$~~~~~~~~~~~~~~~}\\
% \raisebox{0.em}{\rotatebox{90}{\small Log [Acc (\%)] }} 
% \includegraphics[width=0.475\textwidth]{fig/l_inf_case.pdf} \\
% % \vspace{-1\baselineskip}
% \small $\ell_\infty$ radius $\times$ 255~~~~~~~~~~~~~~~\\
% \end{tabular}
% \caption{
% Results of $\ell_\infty$ verification on CIFAR-10, 
% on models trained with Gaussian noise data augmentation with different variances $\sigma_0$. 
% Our method obtains consistently better results.}
% % \vspace{-10pt}
% \label{fig:linf}
% \end{figure}

\begin{table}[htbp]
	\begin{center}
% 	\begin{footnotesize}
    \begin{sc} 
    \setlength\tabcolsep{4pt}
		\begin{tabular}{l|cccccc}
			\toprule
            $l_\infty$ Radius & $2/255$ & $4/255$ & $6/255$ & $8/255$ & $10/255$ & $12/255$\\
            \midrule
            Baseline (\%) & 58 & 42 & 31 & 25 & 18 & 13\\
            Ours (\%) & \bf{60} & \bf{47} & \bf{38} & \bf{32} & \bf{23} & \bf{17}\\
            \bottomrule
		\end{tabular}
		\end{sc}
    % \end{footnotesize}
	\end{center}
	\caption{\label{table:linf-CIFAR-10} Certified top-1 accuracy of the best classifiers with various $l_\infty$ radius on CIFAR-10.}
	%We use the same model as \cite{salman2019provably} and do not train any new models. }
% 	\vspace{-10pt}
\end{table}

% \vspace{-10pt}
% \red
{To further confirm the advantage of our method, we plot in Fig.\ref{fig:linf}
the certified accuracy of our method and Gaussian baseline using models trained with Gaussian perturbation of different variances $\sigma_0$ under different $\ell_\infty$ radius.} Our approach outperforms baseline consistently, especially when the $\ell_\infty$ radius is large. 
We also experimented our method and baseline on ImageNet but did not obtain non-trivial results. This is because $\ell_\infty$ verification is extremely hard with very large dimensions \cite{kumar2020curse, blum2020random}.
Future work will investigate how to obtain non-trivial bounds for $\ell_\infty$ attacking at ImageNet scales with smoothing classifiers.

% \vspace{-10pt}
\section{Conclusion} 
% \vspace{-10pt}
We propose a general \functional optimization based framework of adversarial certification with non-Gaussian smoothing distributions. Based on the insights from our new framework and high dimensional geometry, we propose a new family of non-Gaussian smoothing distributions, which outperform the Gaussian and Laplace smoothing for certifying $\ell_1$, $\ell_2$ and $\ell_\infty$ attacking. Our work provides a basis for a variety of future directions, including improved methods for $\ell_p$ attacks, and tighter bounds based on adding additional constraints to our optimization framework.

%our bound. 
%better results for 
%We then propose a distribution family for improving the certification accuracy for perturbation defined in a $\ell_\infty$ and a $\ell_2$ region.
%We demonstrate through experiments that our proposed distributions can consistently outperform Gaussian distributions in large scale datasets for both $\ell_2$ and $\ell_\infty$ perturbation.
%For future works, 
%we plan to study 
%In the future, we will first consider perturbation defined in more general $\ell_p$ region.
% Second, we will develop more sampling-efficient distributions.
%Second, we will try to introduce more constrains to the optimization objective in \eqref{eq:objective2}. For example,  first-order gradient information of the base classifier or other forms of \functional space $\mathcal{F}$ can be included to the constraint, which may lead to theoretically better lower bound. Evaluating more points rather than just evaluate $\ft$ on $\x$ is also a promising direction, which gives more constrain for the \functional optimization. Furthermore, training models with noise from our proposed distributions will intuitively gives better results than using Gaussian pre-trained model, but is also left as future work due to computational limits.
% , \emph{e.g.} multiple points constraint or first-order gradient information.
% and $\ell_1$ (as $\B_1^{r} \subseteq \B_2^{r}$, the goal is to beat cohen's $\ell_2$ rate under same radius. I believe something like $\pi_{\0}(\z) \propto \Vert\z\Vert_1^{-k}\exp(\frac{-\Vert\z\Vert_2^2}{2\sigma_r^2})$ will work. --narsil)

\section*{Broader Impact}
Adversarial certification via randomized smoothing could achieve \textit{guaranteed} robust machine learning models, thus has wide application on AI security. a \& b) With our empirical results, security engineers could get better performance on defending against vicious attacks; With our theoretical results, it will be easier for following researchers to derive new bounds for different kinds of smoothing methods. We don't foresee the possibility that it could bring negative social impacts. c) Our framework is mathematically rigorous thus  
would never fail. d) Our method doesn't have bias in data as we provide a general certification method for all tasks and data, and our distribution is not adaptive towards data.

\section*{Acknowledgement}
This work is supported by Beijing Nova Program (No. 202072) from Beijing Municipal Science \& Technology Commission.
This work is also supported in part by NSF CAREER 1846421. We would like to thank Tongzheng Ren, Jiaye Teng, Yang Yuan and the reviewers for helpful suggestions that improved the paper.

% Authors are required to include a statement of the broader impact of their work, including its ethical aspects and future societal consequences. 
% Authors should discuss both positive and negative outcomes, if any. For instance, authors should discuss a) 
% who may benefit from this research, b) who may be put at disadvantage from this research, c) what are the consequences of failure of the system, and d) whether the task/method leverages
% biases in the data. If authors believe this is not applicable to them, authors can simply state this.

% Use unnumbered first level headings for this section, which should go at the end of the paper. {\bf Note that this section does not count towards the eight pages of content that are allowed.}

% \newpage
\bibliography{example_paper}

\begin{thebibliography}{10}

\bibitem{kannan2018adversarial}
Harini Kannan, Alexey Kurakin, and Ian Goodfellow.
\newblock Adversarial logit pairing.
\newblock {\em arXiv preprint arXiv:1803.06373}, 2018.

\bibitem{zhang2019defense}
Haichao Zhang and Jianyu Wang.
\newblock Defense against adversarial attacks using feature scattering-based
  adversarial training, 2019.

\bibitem{zhai2019adversarially}
Runtian Zhai, Tianle Cai, Di~He, Chen Dan, Kun He, John Hopcroft, and Liwei
  Wang.
\newblock Adversarially robust generalization just requires more unlabeled
  data.
\newblock {\em arXiv preprint arXiv:1906.00555}, 2019.

\bibitem{athalye2018obfuscated}
Anish Athalye, Nicholas Carlini, and David~A. Wagner.
\newblock Obfuscated gradients give a false sense of security: Circumventing
  defenses to adversarial examples.
\newblock pages 274--283, 2018.

\bibitem{madry2017towards}
Aleksander Madry, Aleksandar Makelov, Ludwig Schmidt, Dimitris Tsipras, and
  Adrian Vladu.
\newblock Towards deep learning models resistant to adversarial attacks.
\newblock 2018.

\bibitem{zhang2019you}
Dinghuai Zhang, Tianyuan Zhang, Yiping Lu, Zhanxing Zhu, and Bin Dong.
\newblock You only propagate once: Accelerating adversarial training via
  maximal principle.
\newblock {\em Advances in neural information processing systems (NeurIPS)},
  2019.

\bibitem{wang2019improving}
Dilin Wang, Chengyue Gong, and Qiang Liu.
\newblock Improving neural language modeling via adversarial training.
\newblock pages 6555--6565, 2019.

\bibitem{lecuyer2018certified}
Mathias Lecuyer, Vaggelis Atlidakis, Roxana Geambasu, Daniel Hsu, and Suman
  Jana.
\newblock Certified robustness to adversarial examples with differential
  privacy.
\newblock {\em arXiv preprint arXiv:1802.03471}, 2018.

\bibitem{cohen2019certified}
Jeremy~M Cohen, Elan Rosenfeld, and J~Zico Kolter.
\newblock Certified adversarial robustness via randomized smoothing.
\newblock {\em arXiv preprint arXiv:1902.02918}, 2019.

\bibitem{salman2019provably}
Hadi Salman, Greg Yang, Jerry Li, Pengchuan Zhang, Huan Zhang, Ilya
  Razenshteyn, and Sebastien Bubeck.
\newblock Provably robust deep learning via adversarially trained smoothed
  classifiers.
\newblock {\em arXiv preprint arXiv:1906.04584}, 2019.

\bibitem{lee2019stratified}
Guang-He Lee, Yang Yuan, Shiyu Chang, and Tommi~S Jaakkola.
\newblock A stratified approach to robustness for randomly smoothed
  classifiers.
\newblock {\em Advances in neural information processing systems (NeurIPS)},
  2019.

\bibitem{li2018second}
Bai Li, Changyou Chen, Wenlin Wang, and Lawrence Carin.
\newblock Second-order adversarial attack and certifiable robustness.
\newblock {\em Advances in neural information processing systems (NeurIPS)},
  2019.

\bibitem{dvijotham2020a}
Krishnamurthy~(Dj) Dvijotham, Jamie Hayes, Borja Balle, Zico Kolter, Chongli
  Qin, Andras Gyorgy, Kai Xiao, Sven Gowal, and Pushmeet Kohli.
\newblock A framework for robustness certification of smoothed classifiers
  using f-divergences.
\newblock In {\em International Conference on Learning Representations}, 2020.

\bibitem{teng2020ell}
Jiaye Teng, Guang-He Lee, and Yang Yuan.
\newblock {\$}{\textbackslash}ell{\_}1{\$} adversarial robustness certificates:
  a randomized smoothing approach, 2020.

\bibitem{jia2020certified}
Jinyuan Jia, Xiaoyu Cao, Binghui Wang, and Neil~Zhenqiang Gong.
\newblock Certified robustness for top-k predictions against adversarial
  perturbations via randomized smoothing.
\newblock In {\em International Conference on Learning Representations}, 2020.

\bibitem{wong2017provable}
Eric Wong and J~Zico Kolter.
\newblock Provable defenses against adversarial examples via the convex outer
  adversarial polytope.
\newblock {\em arXiv preprint arXiv:1711.00851}, 2017.

\bibitem{DBLP:conf/uai/DvijothamSGMK18}
Krishnamurthy Dvijotham, Robert Stanforth, Sven Gowal, Timothy~A. Mann, and
  Pushmeet Kohli.
\newblock A dual approach to scalable verification of deep networks.
\newblock In {\em Proceedings of the Thirty-Fourth Conference on Uncertainty in
  Artificial Intelligence, {UAI} 2018, Monterey, California, USA, August 6-10,
  2018}, pages 550--559, 2018.

\bibitem{jordan2019provable}
Matt Jordan, Justin Lewis, and Alexandros~G Dimakis.
\newblock Provable certificates for adversarial examples: Fitting a ball in the
  union of polytopes.
\newblock {\em arXiv preprint arXiv:1903.08778}, 2019.

\bibitem{deng2009imagenet}
Jia Deng, Wei Dong, Richard Socher, Li-Jia Li, Kai Li, and Li~Fei-Fei.
\newblock Imagenet: A large-scale hierarchical image database.
\newblock In {\em 2009 IEEE conference on computer vision and pattern
  recognition}, pages 248--255. Ieee, 2009.

\bibitem{vershynin2018high}
Roman Vershynin.
\newblock {\em High-dimensional probability: An introduction with applications
  in data science}, volume~47.
\newblock Cambridge University Press, 2018.

\bibitem{carlini2017provably}
Nicholas Carlini, Guy Katz, Clark Barrett, and David~L. Dill.
\newblock Provably minimally-distorted adversarial examples, 2017.

\bibitem{ehlers2017formal}
Ruediger Ehlers.
\newblock Formal verification of piece-wise linear feed-forward neural
  networks.
\newblock In {\em International Symposium on Automated Technology for
  Verification and Analysis}, pages 269--286. Springer, 2017.

\bibitem{cheng2017maximum}
Chih{-}Hong Cheng, Georg N{\"{u}}hrenberg, and Harald Ruess.
\newblock Maximum resilience of artificial neural networks.
\newblock In {\em Automated Technology for Verification and Analysis - 15th
  International Symposium, {ATVA} 2017, Pune, India, October 3-6, 2017,
  Proceedings}, pages 251--268, 2017.

\bibitem{dutta2017output}
Souradeep Dutta, Susmit Jha, Sriram Sankaranarayanan, and Ashish Tiwari.
\newblock Output range analysis for deep feedforward neural networks, 2018.

\bibitem{raghunathan2018semidefinite}
Aditi Raghunathan, Jacob Steinhardt, and Percy Liang.
\newblock Semidefinite relaxations for certifying robustness to adversarial
  examples, 2018.

\bibitem{zhang2018efficient}
Huan Zhang, Tsui-Wei Weng, Pin-Yu Chen, Cho-Jui Hsieh, and Luca Daniel.
\newblock Efficient neural network robustness certification with general
  activation functions.
\newblock In {\em Advances in neural information processing systems}, pages
  4939--4948, 2018.

\bibitem{xie2017mitigating}
Cihang Xie, Jianyu Wang, Zhishuai Zhang, Zhou Ren, and Alan Yuille.
\newblock Mitigating adversarial effects through randomization, 2018.

\bibitem{liu2018towards}
Xuanqing Liu, Minhao Cheng, Huan Zhang, and Cho-Jui Hsieh.
\newblock Towards robust neural networks via random self-ensemble.
\newblock In {\em Proceedings of the European Conference on Computer Vision
  (ECCV)}, pages 369--385, 2018.

\bibitem{li1998some}
Wenbo~V Li and James Kuelbs.
\newblock Some shift inequalities for gaussian measures.
\newblock In {\em High dimensional probability}, pages 233--243. Springer,
  1998.

\bibitem{yang2020randomized}
Greg Yang, Tony Duan, Edward Hu, Hadi Salman, Ilya Razenshteyn, and Jerry Li.
\newblock Randomized smoothing of all shapes and sizes.
\newblock {\em arXiv preprint arXiv:2002.08118}, 2020.

\bibitem{blum2020random}
Avrim Blum, Travis Dick, Naren Manoj, and Hongyang Zhang.
\newblock Random smoothing might be unable to certify linf robustness for
  high-dimensional images.
\newblock {\em arXiv preprint arXiv:2002.03517}, 2020.

\bibitem{kumar2020curse}
Aounon Kumar, Alexander Levine, Tom Goldstein, and Soheil Feizi.
\newblock Curse of dimensionality on randomized smoothing for certifiable
  robustness.
\newblock {\em arXiv preprint arXiv:2002.03239}, 2020.

\bibitem{Boyd:2004:CO:993483}
Stephen Boyd and Lieven Vandenberghe.
\newblock {\em Convex Optimization}.
\newblock Cambridge University Press, New York, NY, USA, 2004.

\end{thebibliography}
\bibliographystyle{unsrt}

\newpage

\appendix
% \onecolumn

\section{Proofs}
\subsection{Proof for Theorem 1}
\label{sec:proof_thm1}
\subsubsection{Proof for (I) and (II)}
First, observe that the constraint in \eqref{eq:constrant_opt} can be equivalently replaced by an inequality constraint $f_{\pizero}(\x) \geq \fpizero(\x)$. 
Therefore, the Lagrangian multiplier can be restricted to be $\lambda \geq 0.$ 
%Denote by $\lambda\geq 0$ %{\color{red} as we use equality constraint, $\lambda\in R^d$}
%the Lagrangian multiplier for the inequality constraint. % in \eqref{eq:constrant_opt}. 
We have 
\begin{equation*}
% \begin{align*}
\begin{split}
 \V(\F,\B)= & \min_{\d\in\mathcal{\B}}\min_{f\in\F}\max_{\lambda\ge 0}\E_{\pi_{\d}}[f(\x+\z)] +\lambda\left(\fpizero(\x)- \E_{\pizero}[f(\x+\z)]\right.)\\
\ge & \max_{\lambda\ge 0}\min_{\d\in\mathcal{\B}}\min_{f\in\F}\E_{\pi_{\d}}[f(\x+\z)] +\lambda\left(\fpizero(\x)-\E_{\pizero}[f(\x+\z)]\right.) \\ 
% \!\!\!\!\!\!\!\ant{\scriptsize exchange $\min$ and $\max$}\\
= & \max_{\lambda\ge 0}\min_{\d\in\mathcal{\B}}
\bigg\{ \lambda\fpizero(\x) ~+~ \min_{f\in\F}  \E_{\pi_{\d}}[f(\x+\z)] - \lambda \E_{\pizero}[f(\x+\z)] )   \bigg\} \\ 
= & \max_{\lambda\ge 0}\min_{\d\in\mathcal{\B}}
\left\{ \lambda \fpizero(\x) ~-~ \D_\F(\lambda \pizero~\Vert~\pi_{\d}) \right.\}.
\end{split}
\end{equation*}
 II) follows a straightforward calculation. 
 
\subsubsection{Proof for (III), the strong duality}
\label{sec:proof_dual}
We first introduce the following lemma, which is a straight forward generalization of the strong Lagrange duality to functional optimization case.
\begin{lemma} \label{lem: strong duality}
Given some $\d^{*}$, we have
\begin{align*}
 & \max_{\lambda\in\R}\min_{f\in\mathcal{F}_{[0,1]}}\mathbb{E}_{\pi_{\d^{*}}}\left[f(\x+\z)\right]+\lambda\left(\ff(\x)-\mathbb{E}_{\pi_{\0}}\left[f(\x+\z)\right]\right)\\
= & \min_{f\in\mathcal{F}_{[0,1]}}\max_{\lambda\in\R}\mathbb{E}_{\pi_{\d^{*}}}\left[f(\x+\z)\right]+\lambda\left(\ff(\x)-\mathbb{E}_{\pi_{\0}}\left[f(\x+\z)\right]\right).
\end{align*}
\end{lemma}
The proof of Lemma \ref{lem: strong duality} is standard. However, for completeness, we include it here.
\begin{proof}
Without loss of generality,
we assume $\ff(\x)\in(0,1)$, otherwise the feasible set is trivial.

Let $\alpha^{*}$ be the value of the optimal solution of the primal
problem. We define $\ff(\x)-\mathbb{E}_{\pi_{\0}}\left[f(\x+\z)\right]=h[f]$
and $g[f]=\mathbb{E}_{\pi_{\d^{*}}}\left[f(\x+\z)\right]$. We
define the following two sets:
\begin{align*}
\mathcal{A} & =\left\{ (v,t)\in\R\times\R:\exists f\in\mathcal{F}_{[0,1]},h[f]=v,g[f]\le t\right\} \\
\mathcal{B} & =\left\{ (0,s)\in\R\times\R:s<\alpha^{*}\right\} .
\end{align*}
Notice that both sets $\mathcal{A}$ and $\mathcal{B}$ are convex.
This is obvious for $\mathcal{B}$. For any $(v_{1},t_{1})\in\mathcal{A}$
and $(v_{2},t_{2})\in\mathcal{A}$, we define $f_{1}\in\mathcal{F}_{[0,1]}$
such that $h[f_{1}]=v_{1},g[f_{1}]\le t_{1}$ (and similarly we define
$f_{2}$). Notice that for any $\gamma\in[0,1]$, we have 
\begin{align*}
\gamma f_{1}+(1-\gamma)f_{2} & \in\mathcal{F}_{[0,1]}\\
\gamma h[f_{1}]+(1-\gamma)h[f_{2}] & =\gamma v_{1}+\text{(1-\ensuremath{\gamma})}v_{2}\\
\gamma g[f_{1}]+(1-\gamma)g[f_{2}] & \le\gamma t_{1}+(1-\gamma)t_{2},
\end{align*}
which implies that $\gamma(v_{1},t_{1})+(1-\gamma)(v_{2},t_{2})\in\mathcal{A}$
and thus $\mathcal{A}$ is convex. Also notice that by definition,
$\mathcal{A}\cap\mathcal{B}=\emptyset$. Using separating hyperplane
theorem, there exists a point $(q_{1},q_{2})\neq(0,0)$ and a value
$\alpha$ such that for any $(v,t)\in\mathcal{A}$, $q_{1}v+q_{2}t\ge\alpha$
and for any $(0,s)\in\mathcal{B}$, $q_{2}s\le\alpha$. Notice that
we must have $q_{2}\ge0$, otherwise, for sufficient $s$, we will
have $q_{2}s>\alpha$. We thus have, for any $f\in\mathcal{F}_{[0,1]}$,
we have 
\[
q_{1}h[f]+q_{2}g[f]\ge\alpha^{*}\ge q_{2}\alpha^{*}.
\]
If $q_{2}>0$, we have 
\[
\max_{\lambda\in\R}\min_{f\in\mathcal{F}_{[0,1]}}g[f]+\lambda h[f]\ge\min_{f\in\mathcal{F}_{[0,1]}}g[f]+\frac{q_{1}}{q_{2}}h[f]\ge\alpha^{*},
\]
which gives the strong duality. If $q_{2}=0$, we have for any $f\in\mathcal{F}_{[0,1]}$,
$q_{1}h[f]\ge0$ and by the separating hyperplane theorem, $q_{1}\neq0$.
However, this case is impossible: If $q_{1}>0$, choosing $f\equiv1$
gives $q_{1}h[f]=q_{1}\left(\ff(\x)-1\right)<0$; If $q_{1}<0$,
by choosing $f\equiv0$, we have $q_{1}h[f]=q_{1}\left(\ff(\x)-0\right)<0$.
Both cases give contradiction.

\end{proof}

Based on Lemma \ref{lem: strong duality}, we have the proof of the strong duality as follows.

Notice that by Lagrange multiplier method, our primal problem can be
rewritten as follows: 
\[
\min_{\d\in\B}\min_{f\in\mathcal{F}_{[0,1]}}\max_{\lambda\in\mathbb{R}}\mathbb{E}_{\pi_{\d}}\left[f(\x+\z)\right]+\lambda\left(\ff(\x)-\mathbb{E}_{\pi_{\0}}\left[f(\x+\z)\right]\right),
\]
and the dual problem is 
\begin{align*}
 & \max_{\lambda\in\mathbb{R}}\min_{\d\in\B}\min_{f\in\mathcal{F}_{[0,1]}}\mathbb{E}_{\pi_{\d}}\left[f(\x+\z)\right]+\lambda\left(\ff(\x)-\mathbb{E}_{\pi_{\0}}\left[f(\x+\z)\right]\right)\\
= & \max_{\lambda\ge0}\min_{\d\in\B}\min_{f\in\mathcal{F}_{[0,1]}}\mathbb{E}_{\pi_{\d}}\left[f(\x+\z)\right]+\lambda\left(\ff(\x)-\mathbb{E}_{\pi_{\0}}\left[f(\x+\z)\right]\right).
\end{align*}
By the assumption that for any $\lambda\ge0$, we have 
\begin{align*}
 & \max_{\lambda\ge0}\min_{\d\in\B}\min_{f\in\mathcal{F}_{[0,1]}}\mathbb{E}_{\pi_{\d}}\left[f(\x+\z)\right]+\lambda\left(\ff(\x)-\mathbb{E}_{\pi_{\0}}\left[f(\x+\z)\right]\right)\\
= & \max_{\lambda\ge0}\min_{f\in\mathcal{F}_{[0,1]}}\mathbb{E}_{\pi_{\d^{*}}}\left[f(\x+\z)\right]+\lambda\left(\ff(\x)-\mathbb{E}_{\pi_{\0}}\left[f(\x+\z)\right]\right),
\end{align*}
for some $\d^{*}\in\boldsymbol{\mathcal{B}}$. We have 
\begin{align*}
 & \max_{\lambda\in\mathbb{R}}\min_{\d\in\B}\min_{f\in\mathcal{F}_{[0,1]}}\mathbb{E}_{\pi_{\d}}\left[f(\x+\z)\right]+\lambda\left(\ff(\x)-\mathbb{E}_{\pi_{\0}}\left[f(\x+\z)\right]\right)\\
= & \max_{\lambda\ge0}\min_{f\in\mathcal{F}_{[0,1]}}\mathbb{E}_{\pi_{\d^{*}}}\left[f(\x+\z)\right]+\lambda\left(\ff(\x)-\mathbb{E}_{\pi_{\0}}\left[f(\x+\z)\right]\right)\\
= & \max_{\lambda\in\R}\min_{f\in\mathcal{F}_{[0,1]}}\mathbb{E}_{\pi_{\d^{*}}}\left[f(\x+\z)\right]+\lambda\left(\ff(\x)-\mathbb{E}_{\pi_{\0}}\left[f(\x+\z)\right]\right)\\
\overset{*}{=} & \min_{f\in\mathcal{F}_{[0,1]}}\max_{\lambda\in\R}\mathbb{E}_{\pi_{\d^{*}}}\left[f(\x+\z)\right]+\lambda\left(\ff(\x)-\mathbb{E}_{\pi_{\0}}\left[f(\x+\z)\right]\right)\\
\ge & \min_{\d\in\B}\min_{f\in\mathcal{F}_{[0,1]}}\max_{\lambda\in\R}\mathbb{E}_{\pi_{\d^{*}}}\left[f(\x+\z)\right]+\lambda\left(\ff(\x)-\mathbb{E}_{\pi_{\0}}\left[f(\x+\z)\right]\right),
\end{align*}
where the second equality ({*}) is by Lemma \ref{lem: strong duality}.

\subsection{Proof for Corollary~\ref{thm:jiaye}}
\label{sec: retrievejiaye}
\begin{proof}
Given our confidence lower bound
$$
\max_{\lambda \ge 0}\min_{\|\d\|_1\le r}
\left\{\lambda p_0 -\int  \left(\lambda \pi_{\0}(z)-
\pi_{\d}(z)\right)_{+} d z \right\},
$$
One can show that the worst case for $\d$ is obtained when $\d^* = (r, 0, \cdots, 0)$ (see following subsection), thus the bound is
$$
\max_{\lambda \ge 0}\left\{\lambda p_0 -\int\frac{1}{2b}\exp\left(-\frac{|z_1|}{b}\right) \left[\lambda - \exp\left(\frac{\mid z_1\mid -|z_1 + r|}{b}\right)\right]_{+} d z_1 \right\}.
$$
Denote $a$ to be the solution of $\lambda = \exp\left(\frac{|a| - |a+r|}{b}\right)$,  then obviously we have 
$$
a = 
\begin{cases}
-\infty,& b\log\lambda\ge r\\
-\frac{1}{2}\left(b\log\lambda+r\right),& -r< b\log\lambda< r\\
+\infty. & b\log\lambda\le -r
\end{cases}
$$
So the bound above is 
$$
\lambda\int_{z_1>a}\frac{1}{2b}\exp\left(-\frac{|z_1|}{b}\right) d z_1 - \int_{z_1>a}\frac{1}{2b}\exp\left(-\frac{|z_1 + r|}{b}\right) d z_1.
$$

i) $b\log\lambda\ge r \Leftrightarrow \lambda \ge \exp\left(\frac{r}{b}\right) $ \\
the bound is $$
\max_{\lambda\ge e^{r/b}}\left\{\lambda p_0 - (\lambda - 1)\right\} = 1 - \exp\left(\frac{r}{b}\right)\left(1-p_0\right).
$$

ii) $ -r< b\log\lambda< r \Leftrightarrow \exp\left(-\frac{r}{b}\right) < \lambda < \exp\left(\frac{r}{b}\right) $\\
the bound is 
\begin{align*}
    &\max_{\lambda}\left\{\lambda p_0 - \lambda\left[1 - \frac{1}{2}\exp\left(-\frac{b\log\lambda + r}{2b}\right)\right] + \frac{1}{2}\exp\left(\frac{b\log\lambda - r}{2b}\right)\right\} \\
    = &  \max_{\lambda}\left\{\lambda(p_0 - 1) + \frac{\lambda}{2}\exp\left(-\frac{b\log\lambda + r}{2b}\right) + \frac{1}{2}\exp\left(\frac{b\log\lambda - r}{2b}\right)\right\} \\
    =& \frac{1}{2}\exp\left(-\log\left[2(1 - p_0)\right]-\frac{r}{b}\right).
\end{align*}
the extremum is achieved when $\hat\lambda = \exp\left(-2\log\left[2(1 - p_0)\right]-\frac{r}{b}\right)$. Notice that $\hat\lambda$ does not necessarily locate in $\left(e^{-r/b}, e^{r/b}\right)$, so the actual bound is always equal or less than $\frac{1}{2}\exp\left(-\log\left[2(1 - p_0)\right]-\frac{r}{b}\right)$.

iii) $b\log\lambda\le -r \Leftrightarrow \lambda \le \exp\left(-\frac{r}{b}\right) $ \\
the bound is 
$$
\max_{\lambda\le \exp\left(-\frac{r}{b}\right)}\lambda\cdot p_0 = p_0 \exp\left(-\frac{r}{b}\right).
$$

Since $\hat\lambda > e^{r/b} \Leftrightarrow p_0 > 1 - \frac{1}{2}\exp(-\frac{r}{b})$, notice that the lower bound is a concave function w.r.t. $\lambda$, making the final lower bound become
$$
\begin{cases}
1 - \exp\left(\frac{r}{b}\right)\left(1-p_0\right), & \text{when} \quad p_0 > 1 - \frac{1}{2}\exp(-\frac{r}{b}) \\
\frac{1}{2}\exp\left(-\log\left[2(1 - p_0)\right]-\frac{r}{b}\right). & \text{otherwise}
\end{cases}
$$
\end{proof}

\paragraph{Remark} Actually, we have $1 - \exp\left(\frac{r}{b}\right)\left(1-p_0\right) \le \frac{1}{2}\exp\left(-\log\left[2(1 - p_0)\right]-\frac{r}{b}\right) $ all the time.
Another interesting thing is that both the bound can lead to the same radius bound:
\begin{align*}
    1 - \exp\left(\frac{r}{b}\right)\left(1-p_0\right) > \frac{1}{2} &\Leftrightarrow r < -b\log\left[2(1 - p_0)\right] \\
    \frac{1}{2}\exp\left(-\log\left[2(1 - p_0)\right]-\frac{r}{b}\right) > \frac{1}{2} &\Leftrightarrow r < -b\log\left[2(1 - p_0)\right] 
\end{align*}

\subsection{Proof for Corollary~\ref{thm:cohen}}
\label{sec: retrievecohen}

\begin{proof}
With strong duality, our confidence lower bound is
$$
\min_{\|\d\|_2\le r}\max_{\lambda \ge 0}
 \left\{\lambda p_0 -\int  \left(\lambda \pi_{\0}(z)-
\pi_{\d}(z)\right)_{+} d z \right\},
$$
define $C_\lambda = \{z: \lambda \pi_{\0}(z)\ge
\pi_{\d}(z)\} = \{z: {\d}^\top z \le \frac{\|{\d}\|^2}{2} + \sigma^2\ln\lambda\}$ and $\Phi(\cdot)$ to be the cdf of standard gaussian distribution, then
\begin{align*}
    &\int \left(\lambda \pi_{\0}(z)- \pi_{{\d}}(z)\right)_{+} d z \\
    =& \int_{C_\lambda}   \left(\lambda \pi_{\0}(z)- \pi_{\d}(z)\right) d z \\
    =& \lambda\cdot\mathbb{P}\left(N(z;\boldsymbol{0}, \sigma^2  \mI)\in C_\lambda\right) - \mathbb{P}\left(N(z;{\d}, \sigma^2  \mI)\in C_\lambda\right) \\
    =& \lambda\cdot \Phi\left(\frac{\|\d\|_2}{2\sigma}+\frac{\sigma\ln\lambda}{\|{\d}\|_2}\right) - \Phi\left(\frac{-\|{\d}\|_2}{2\sigma}+\frac{\sigma\ln\lambda}{\|{\d}\|_2}\right).
\end{align*}
Define 
$$F(\d, \lambda) := \lambda p_0 - \int \left(\lambda \pi_{\0}(z)- \pi_{\d}(z)\right)_{+} d z =
\lambda p_0 - \lambda\cdot \Phi\left(\frac{\|{\d}\|_2}{2\sigma}+\frac{\sigma\ln\lambda}{\|{\d}\|_2}\right) + \Phi\left(\frac{-\|{\d}\|_2}{2\sigma}+\frac{\sigma\ln\lambda}{\|{\d}\|_2}\right).$$
For $\forall \d$, $F$ is a concave function w.r.t. $\lambda$, as $F$ is actually a summation of many concave piece wise linear function. See \cite{Boyd:2004:CO:993483} for more discussions  of properties of concave functions.

Define $\hat\lambda_{\d} = \exp\left(\frac{2\sigma\|\d\|_2\Phi^{-1}(p_0)-\|\d\|_2^2 }{2\sigma^2}\right)$, simple calculation can show $\frac{\partial F(\d, \lambda)}{\partial \lambda}|_{\lambda=\hat\lambda_{\d}} = 0$, which means

\begin{align*}
\min_{\|\d\|_2\le r}\max_{\lambda \ge 0} F(\d, \lambda) &= \min_{\|\d\|_2\le r} F(\d, \lambda_{\d}) \\
&= \min_{\|\d\|_2\le r} \left\{0 + \Phi\left(\frac{-\|\d\|_2}{2\sigma}+\frac{\sigma\ln\hat\lambda_{\d}}{\|\d\|_2}\right) \right\}\\
&= \min_{\|\d\|_2\le r} \Phi\left(\Phi^{-1}(p_0) - \frac{\|\d\|_2}{\sigma}\right) \\
&= \Phi\left(\Phi^{-1}(p_0) - \frac{r}{\sigma}\right)
\end{align*}
This tells us 
$$
\min_{\|\d\|_2\le r}\max_{\lambda \ge 0} F(\d, \lambda) > 1/2 \Leftrightarrow \Phi\left(\Phi^{-1}(p_0) - \frac{r}{\sigma}\right) > 1/2 \Leftrightarrow r < \sigma\cdot\Phi^{-1}(p_0),
$$
i.e. the certification radius is $\sigma\cdot\Phi^{-1}(p_0)$. This is exactly the core theoretical contribution of \cite{cohen2019certified}. This bound has a straight forward expansion for multi-class classification situations, we refer interesting readers to Appendix~\ref{sec:bilateral}.
\end{proof}

\subsection {Proof For Theorem \ref{thm:opt_mu_l2} and \ref{thm:opt_mu_linf}}
\label{sec: opt_mu_general_proof}

\subsubsection{Proof for $\ell_2$ and $\ell_\infty$ cases}
Here we consider a more general smooth distribution $\pi_{\boldsymbol{0}}(\z)\propto\left\Vert \z\right\Vert _{\infty}^{-k_{1}}\left\Vert \z\right\Vert _{2}^{-k_{2}}\exp\left(-\frac{\left\Vert \z\right\Vert _{2}^{2}}{2\sigma^{2}}\right)$, for some $k_{1},k_{2}\ge0$ and $\sigma>0$. We first gives the following key theorem shows that $\mathbb{D}_{\mathcal{F}_{[0,1]}}\left(\lambda\pi_{\0}\parallel\pi_{\d}\right)$ increases as $\left|\delta_{i}\right|$ becomes larger for every dimension $i$.

\begin{theorem} \label{thm:opt_mu_general}
Suppose $\pi_{\boldsymbol{0}}(\z)\propto\left\Vert \z\right\Vert _{\infty}^{-k_{1}}\left\Vert \z\right\Vert _{2}^{-k_{2}}\exp\left(-\frac{\left\Vert \z\right\Vert _{2}^{2}}{2\sigma^{2}}\right)$, for some $k_{1},k_{2}\ge0$ and $\sigma>0$, for any $\lambda\ge 0$ we have 
\[
\mathrm{sgn}(\delta_{i})\frac{\partial}{\partial\delta_{i}}\mathbb{D}_{\mathcal{F}_{[0,1]}}\left(\lambda\pi_{\0}\parallel\pi_{\d}\right)\ge0,
\]for any $i\in\{1,2,...,d\}$. 
\end{theorem}
Theorem \ref{thm:opt_mu_l2} and \ref{thm:opt_mu_linf} directly follows the above theorem. Notice that in Theorem \ref{thm:opt_mu_l2}, as our distribution is spherical symmetry, it is equivalent to set $\boldsymbol{\mathcal{B}}=\left\{ \d:\d=[a,0,...,0]^{\top},a\le r\right\}$ by rotating the axis. 

\begin{proof}

Given $\lambda$, $k_{1}$ and $k_{2}$, we define $\phi_{1}(s)=s^{-k_{1}}$,
$\phi_{2}(s)=s^{-k_{2}}e^{-\frac{s^{2}}{\sigma^{2}}}$. Notice
that $\phi_{1}$ and $\phi_{2}$ are monotone decreasing for non-negative
$s$. By the symmetry, without loss of generality, we assume $\d=[\delta_{1},...,\delta_{d}]^{\top}$
for $\delta_{i}\ge0$, $i\in[d]$. Notice that
\begin{align*}
\frac{\partial}{\partial\delta_{i}}\left\Vert \x-\d\right\Vert _{\infty} & \mathbb{=I}\{\left\Vert \x-\d\right\Vert _{\infty}=\left|x_{i}-\delta_{i}\right|\}\frac{\partial}{\partial\delta_{i}}\sqrt{\left(x_{i}-\delta_{i}\right)^{2}}\\
 & =\mathbb{I}\{\left\Vert \x-\d\right\Vert _{\infty}=\left|x_{i}-\delta_{i}\right|\}\frac{-\left(x_{i}-\delta_{i}\right)}{\left\Vert \x-\d\right\Vert _{\infty}}.
\end{align*}
And also
\begin{align*}
\frac{\partial}{\partial\delta_{i}}\left\Vert \x-\u\right\Vert _{2} & =\frac{\partial}{\partial\delta_{i}}\sqrt{\sum_{i}\left(x_{i}-\mu_{i}\right)^{2}}\\
 & =\frac{-\left(x_{i}-\mu_{i}\right)}{\left\Vert \x-\u\right\Vert _{2}}.
\end{align*}
We thus have
\begin{align*}
 & \frac{\partial}{\partial\delta_{1}}\int\left(\lambda\pi_{\boldsymbol{0}}(\x)-\pi_{\d}(\x)\right)_{+}d\x\\
= & -\int\mathbb{I}\left\{ \lambda\pi_{\boldsymbol{0}}(\x)\ge\pi_{\d}(\x)\right\} \frac{\partial}{\partial\delta_{1}}\pi_{\d}(\x)d\x\\
= & \int\mathbb{I}\left\{ \lambda\pi_{\boldsymbol{0}}(\x)\ge\pi_{\d}(\x)\right\} F_{1}\left(\left\Vert \x-\d\right\Vert _{\infty},\left\Vert \x-\d\right\Vert _{2}\right)d\x\\
= & \int\mathbb{I}\left\{ \lambda\pi_{\boldsymbol{0}}(\x)\ge\pi_{\d}(\x),x_{1}>\delta_{1}\right\} F_{1}\left(\left\Vert \x-\d\right\Vert _{\infty},\left\Vert \x-\d\right\Vert _{2}\right)d\x\\
+ & \int\mathbb{I}\left\{ \lambda\pi_{\boldsymbol{0}}(\x)\ge\pi_{\d}(\x),x_{1}<\delta_{1}\right\} F_{1}\left(\left\Vert \x-\d\right\Vert _{\infty},\left\Vert \x-\d\right\Vert _{2}\right)d\x,
\end{align*}
where we define
\begin{align*}
 & F_{1}\left(\left\Vert \x-\d\right\Vert _{\infty},\left\Vert \x-\d\right\Vert _{2}\right)\\
 & =\phi'_{1}\left(\left\Vert \x-\d\right\Vert _{\infty}\right)\phi_{2}\left(\left\Vert \x-\d\right\Vert _{2}\right)\mathbb{I}\{\left\Vert \x-\d\right\Vert _{\infty}=\left|x_{1}-\delta_{1}\right|\}\frac{\left(x_{1}-\delta_{1}\right)}{\left\Vert \x-\d\right\Vert _{\infty}}\\
 & +\phi_{1}\left(\left\Vert \x-\d\right\Vert _{\infty}\right)\phi'_{2}\left(\left\Vert \x-\d\right\Vert _{2}\right)\frac{\left(x_{1}-\delta_{1}\right)}{\left\Vert \x-\d\right\Vert _{2}}.
\end{align*}
Notice that as $\phi_{1}'\le0$ and $\phi'_{2}\le0$ and we have 
\begin{align*}
\int\mathbb{I}\left\{ \lambda\pi_{\boldsymbol{0}}(\x)\ge\pi_{\d}(\x),x_{1}>\delta_{1}\right\} F_{1}\left(\left\Vert \x-\d\right\Vert _{\infty},\left\Vert \x-\d\right\Vert _{2}\right)d\x & \le0\\
\int\mathbb{I}\left\{ \lambda\pi_{\boldsymbol{0}}(\x)\ge\pi_{\d}(\x),x_{1}<\delta_{1}\right\} F_{1}\left(\left\Vert \x-\d\right\Vert _{\infty},\left\Vert \x-\d\right\Vert _{2}\right)d\x & \ge0.
\end{align*}
Our target is to prove that $\frac{\partial}{\partial\delta_{1}}\int\left(\lambda\pi_{\boldsymbol{0}}(\x)-\pi_{\d}(\x)\right)_{+}d\x\ge0$.
Now define the set 
\begin{align*}
H_{1} & =\left\{ \x:\lambda\pi_{\boldsymbol{0}}(\x)\ge\pi_{\d}(\x),x_{1}>\delta_{1}\right\} \\
H_{2} & =\left\{ [2\delta_{1}-x_{1},x_{2},...,x_{d}]^{\top}:\x=[x_{1},...,x_{d}]^{\top}\in H_{1}\right\} .
\end{align*}
 Here the set $H_{2}$ is defined as a image of a bijection 
\[
\mathrm{proj}(\x)=\left[2\delta_{1}-x_{1},x_{2},...,x_{d}\right]^{\top}=\tilde{\x},
\]
 that is constrained on the set $H_{1}$. Notice that under our definition,
\begin{align*}
 & \int\mathbb{I}\left\{ \lambda\pi_{\boldsymbol{0}}(\x)\ge\pi_{\d}(\x),x_{1}>\delta_{1}\right\} F_{1}\left(\left\Vert \x-\d\right\Vert _{\infty},\left\Vert \x-\d\right\Vert _{2}\right)d\x\\
= & \int_{H_{1}}F_{1}\left(\left\Vert \x-\d\right\Vert _{\infty},\left\Vert \x-\d\right\Vert _{2}\right)d\x.
\end{align*}
Now we prove that 
\begin{align*}
 & \int\mathbb{I}\left\{ \lambda\pi_{\boldsymbol{0}}(\x)\ge\pi_{\d}(\x),x_{1}<\delta_{1}\right\} F_{1}\left(\left\Vert \x-\d\right\Vert _{\infty},\left\Vert \x-\d\right\Vert _{2}\right)d\x\\
\overset{(1)}{\ge} & \int_{H_{2}}F_{1}\left(\left\Vert \x-\d\right\Vert _{\infty},\left\Vert \x-\d\right\Vert _{2}\right)d\x\\
\overset{(2)}{=} & \left|\int_{H_{1}}F_{1}\left(\left\Vert \x-\d\right\Vert _{\infty},\left\Vert \x-\d\right\Vert _{2}\right)d\x\right|.
\end{align*}

\paragraph{Property of the projection}

Before we prove the (1) and (2), we give the following property of
the defined projection function. For any $\tilde{\x}=\mathrm{proj}(\x)$,
$\x\in H_{1}$, we have 
\begin{align*}
\left\Vert \x-\d\right\Vert _{\infty} & =\left\Vert \tilde{\x}-\d\right\Vert _{\infty}\\
\left\Vert \x-\d\right\Vert _{2} & =\left\Vert \tilde{\x}-\d\right\Vert _{2}\\
\left\Vert \x\right\Vert _{2} & \ge\left\Vert \tilde{\x}\right\Vert _{2}\\
\left\Vert \x\right\Vert _{\infty} & \ge\left\Vert \tilde{\x}\right\Vert _{\infty}.
\end{align*}
This is because 
\begin{align*}
\tilde{x}_{i} & =x_{i},i\in[d]-\{1\}\\
\tilde{x}_{1} & =2\delta_{1}-x_{1},
\end{align*}
and by the fact that $x_{1}\ge\delta_{1}\ge0$, we have $\left|\tilde{x}_{1}\right|\le\left|x_{1}\right|$
and $\left|\tilde{x}_{1}-\delta_{1}\right|\le\left|x_{1}-\delta_{1}\right|$.

\paragraph{Proof of Equality (2)}

By the fact that $\mathrm{proj}$ is bijective constrained on the
set $H_{1}$ and the property of $\mathrm{proj}$, we have 
\begin{align*}
 & \int_{H_{2}}F_{1}\left(\left\Vert \tilde{\x}-\d\right\Vert _{\infty},\left\Vert \tilde{\x}-\d\right\Vert _{2}\right)d\tilde{\x}\\
= & \int_{H_{2}}\phi'_{1}\left(\left\Vert \tilde{\x}-\d\right\Vert _{\infty}\right)\phi_{2}\left(\left\Vert \tilde{\x}-\d\right\Vert _{2}\right)\mathbb{I}\{\left\Vert \tilde{\x}-\d\right\Vert _{\infty}=\left|\tilde{x}_{1}-\delta_{1}\right|\}\frac{\left(\tilde{x}_{1}-\delta_{1}\right)}{\left\Vert \tilde{\x}-\d\right\Vert _{\infty}}d\tilde{\x}\\
+ & \int_{H_{2}}\phi_{1}\left(\left\Vert \tilde{\x}-\d\right\Vert _{\infty}\right)\phi'_{2}\left(\left\Vert \tilde{\x}-\d\right\Vert _{2}\right)\frac{\left(\tilde{x}_{1}-\delta_{1}\right)}{\left\Vert \tilde{\x}-\d\right\Vert _{2}}d\tilde{\x}\\
\overset{(*)}{=} & \int_{H_{1}}\phi'_{1}\left(\left\Vert \x-\d\right\Vert _{\infty}\right)\phi_{2}\left(\left\Vert \x-\d\right\Vert _{2}\right)\mathbb{I}\{\left\Vert \x-\d\right\Vert _{\infty}=\left|x_{1}-\delta_{1}\right|\}\frac{\left(\delta_{1}-x_{1}\right)}{\left\Vert \x-\d\right\Vert _{\infty}}\left|\mathrm{det}\left(\boldsymbol{J}\right)\right|d\x\\
+ & \int_{H_{1}}\phi_{1}\left(\left\Vert \x-\d\right\Vert _{\infty}\right)\phi'_{2}\left(\left\Vert \x-\d\right\Vert _{2}\right)\frac{\left(\delta_{1}-x_{1}\right)}{\left\Vert \x-\d\right\Vert _{2}}d\x\\
= & -\int_{H_{1}}F_{1}\left(\left\Vert \x-\d\right\Vert _{\infty},\left\Vert \x-\d\right\Vert _{2}\right)d\x,
\end{align*}
where $(*)$ is by change of variable $\tilde{\x}=\mathrm{proj}(\x)$
and $\boldsymbol{J}$ is the Jacobian matrix $\boldsymbol{J}=\left[\begin{array}{cccc}
-1 & 0 & \cdots & 0\\
0 & 1 & \cdots & 0\\
\vdots & \vdots & \ddots & \vdots\\
0 & 0 & \cdots & 1
\end{array}\right]$ and here we have the fact that $\tilde{x}_{1}-\delta_{1}=(2\delta_{1}-x_{1})-\delta_{1}=-(x_{1}-\delta_{1})$.

\paragraph{Proof of Inequality (1)}

This can be done by verifying that $H_{2}\subseteq\left\{ \x:\lambda\pi_{\boldsymbol{0}}(\x)\ge\pi_{\d}(\x),x_{1}<\delta_{1}\right\} $.
By the property of the projection, for any $\x\in H_{1}$, let $\tilde{\x}=\mathrm{proj}(\x)$,
then $\lambda\pi_{\boldsymbol{0}}(\tilde{\x})\ge\lambda\pi_{\boldsymbol{0}}(\x)\ge\pi_{\d}(\x)=\pi_{\d}(\tilde{\x})$
(by the fact that t $\phi_{1}$ and $\phi_{2}$ are monotone decreasing).
It implies that for any $\tilde{\x}\in H_{2}$, we have $\lambda\pi_{\boldsymbol{0}}(\tilde{\x})\ge\pi_{\d}(\tilde{\x})$
and thus $H_{2}\subseteq\left\{ \x:\pi_{\boldsymbol{0}}(\x)\ge\pi_{\d}(\x),x_{1}<\delta_{1}\right\} $.

\paragraph{Final statement}

By the above result, we have 
\[
\frac{\partial}{\partial\delta_{1}}\int\left(\lambda\pi_{\boldsymbol{0}}(\x)-\pi_{\d}(\x)\right)_{+}d\x\ge0,
\]
and the same result holds for any $\frac{\partial}{\partial\delta_{1}}\int\left(\lambda\pi_{\boldsymbol{0}}(\x)-\pi_{\d}(\x)\right)_{+}d\x,i\in[d]$,
which implies our result.
\end{proof}

\subsubsection{Proof for $\ell_1$ case}
Slightly different for former cases, apart from proving $\frac{\partial}{\partial\delta_{i}}\mathbb{D}_{\mathcal{F}_{[0,1]}}\left(\lambda\pi_{\0}\parallel\pi_{\d}\right)\ge 0$ for $\forall \delta_i \ge 0$, we also need to demonstrate

\begin{theorem}
Suppose $\pi_{\boldsymbol{0}}(\x) \propto \|\x\|^{-k}\exp\left(-\frac{\norm{\x}_1}{b}\right) $, then for $\d=(r, d-r, \delta_3, \delta_4, \cdots)$ and $\tilde{\d} = (0, d , \delta_3, \delta_4, \cdots)$, $0<r<d$, we have
$$\mathbb{D}_{\mathcal{F}_{[0,1]}}\left(\lambda\pi_{\0}\parallel\pi_{\d}\right) \ge \mathbb{D}_{\mathcal{F}_{[0,1]}}\left(\lambda\pi_{\0}\parallel\pi_{\tilde\d}\right)
$$
\end{theorem}
\begin{proof}
We turn to show that 
\[
\frac{\partial}{\partial r}\mathbb{D}_{\mathcal{F}_{[0,1]}}\left(\lambda\pi_{\0}\parallel\pi_{\d}\right)\le 0,
\]
for $\d=(r, d-r, \delta_3, \delta_4, \cdots)$ and $r < d/2$. We define $\phi(s) = s^{-k}\exp(-\frac{s}{b})$. With
$$
\frac{\partial}{\partial\delta_{i}}\norm{\x - \d}_1 = \frac{\partial}{\partial\delta_{i}}|x_i - \delta_i| = -\mathrm{sgn}(x_i - \delta_i) = \frac{\delta_i - x_i}{|x_i - \delta_i|},
$$
We have
\begin{align*}
    &\frac{\partial}{\partial r}\mathbb{D}_{\mathcal{F}_{[0,1]}}\left(\lambda\pi_{\0}\parallel\pi_{\d}\right) \\
    =& - \int \mathbb{I}\left\{ \lambda\pi_{\boldsymbol{0}}(\x)\ge\pi_{\d}(\x)\right\}\frac{\partial}{\partial r}\pi_{\d}(\x) d\x\\
    =& \int \mathbb{I}\left\{ \lambda\pi_{\boldsymbol{0}}(\x)\ge\pi_{\d}(\x)\right\}F(\x) d\x,
\end{align*}
where 
\begin{align*}
F(\x) = &- \frac{\partial}{\partial r} \phi\left(\norm{\x-\d}_1\right) = -\phi'\left(\norm{\x-\d}_1\right) \frac{\partial}{\partial r} \norm{\x-\d}_1 \\
& = \phi'\left(\norm{\x-\d}_1\right) \frac{\partial}{\partial r} \left(|x_1-r| + |x_2 -d +r|\right) \\
& = \phi'\left(\norm{\x-\d}_1\right) \cdot \left(\mathrm{sgn}(x_1-r) + \mathrm{sgn}(d - x_2 -r)\right).
\end{align*}
Thus the original derivative becomes
\begin{align*}
 = & \int \mathbb{I}\left\{ \lambda\pi_{\boldsymbol{0}}(\x)\ge\pi_{\d}(\x), x_1 >r, x_2 < d-r\right\} F(\x) d\x  \\
& +  \int \mathbb{I}\left\{ \lambda\pi_{\boldsymbol{0}}(\x)\ge\pi_{\d}(\x), x_1 >r, x_2 > d-r\right\} F(\x) d\x\\
& +  \int \mathbb{I}\left\{ \lambda\pi_{\boldsymbol{0}}(\x)\ge\pi_{\d}(\x), x_1 <r, x_2 > d-r\right\} F(\x) d\x\\
& +  \int \mathbb{I}\left\{ \lambda\pi_{\boldsymbol{0}}(\x)\ge\pi_{\d}(\x), x_1 <r, x_2 < d-r\right\} F(\x) d\x\\
   = & 2 \int \mathbb{I}\left\{ \lambda\pi_{\boldsymbol{0}}(\x)\ge\pi_{\d}(\x), x_1 >r, x_2 < d-r\right\}\phi'(\norm{\x-\d}_1) d\x  \\
    & -2 \int \mathbb{I}\left\{ \lambda\pi_{\boldsymbol{0}}(\x)\ge\pi_{\d}(\x), x_1 < r, x_2 > d-r\right\}\phi'(\norm{\x-\d}_1) d\x
\end{align*}
We only need to show that 

\begin{align*}
\int \mathbb{I}
\left\{ \lambda 
% &
\pi_{\boldsymbol{0}}(\x) \ge\pi_{\d}(\x), x_1 >r, x_2 < d-r \right\} \phi'(\norm{\x-\d}_1) d\x \ge \\
% & 
\int \mathbb{I}\left\{ \lambda\pi_{\boldsymbol{0}}(\x)\ge\pi_{\d}(\x), x_1 < r, x_2 > d-r \right\}
\phi' (\norm{\x-\d}_1) d\x.
\end{align*}

Notice that $r<d/2$, therefore this can be proved with a similar projection $\x \mapsto \tilde{\x}$:
    $$
    (x_1, x_2, x_3, x_4, \cdots) \mapsto (2r - x_1, 2d- 2r -x_2, x_3, x_4, \cdots)
    $$
and the similar deduction as previous theorem.

\end{proof}

\subsection{Theoretical Demonstration about the Ineffetivity of \eqref{eq:inf_inf_distribution}}
\label{sec:lin_impossible}

% Unfortunately, this seemingly natural choice does not work  efficiently for $\ell_\infty$ attacks (even worse than the $\ell_2$ family \eqref{equ:L2smooth}).  
% This is because the volume of the $\ell_\infty$ ball is in some sense ``too large'' (\emph{e.g.}, compared with the volume of $\ell_2$ ball). 
% As a result, in order to make the probability mass of  \eqref{eq:inf_inf_distribution} in a reasonable scale, one has to choose a very small value of $\sigma$, which makes maximum distance term too large to be practically useful. 
%based on previous results under the $\ell_\infty$ setting.
%However, we have the following proposition:
\begin{theorem} \label{thm: impossible linf}
%Consider a $\ell_{\infty}$ perturbation with 
Consider the adversarial attacks on the $\ell_\infty$ ball 
$\B_{\ell_\infty, r}=\{\d:\left\Vert \d\right\Vert _{\infty}\le r\}$.
Suppose we use the smoothing distribution $\pizero$ in \eqref{eq:inf_inf_distribution} and 
%$\pi_{\0}(\z)\propto\left\Vert \z\right\Vert _{\infty}^{-k}\exp\left(-\frac{\left\Vert \z\right\Vert _{\infty}^{2}}{2\sigma^{2}}\right)$
choose the parameters $(k, \sigma)$ such that

1) $\norm{\z}_\infty$  is stochastic bounded when $\z\sim \pizero$, in that for any $\epsilon >0$, there exists a finite $M>0$ such that $\mathrm{P}_{\pizero}(|\z|>M ) \leq \epsilon$;   

2) the mode of $\norm{\z}_\infty$ under $\pizero$ equals $C r$, where $C$ is some fixed positive constant, 
%$\mathcal{O}(1)$ constant. 

%$\mathrm{mode}(\left\Vert \z\right\Vert _{\infty})=Cr$, $\left\Vert \z\right\Vert _{\infty}=\mathcal{O}_{p}(1)$, where $C$
%is some $\mathcal{O}(1)$ constant, 
then for {any $\epsilon\in(0,1)$ and sufficiently large dimension $d$, there exists a constant $t > 1$, such that }, we have 
\[
\max_{\d\in \B_{\ell_\infty, r}}\bigg\{ \mathbb{D}_{\mathcal{F}_{[0,1]}}\left(\lambda\pi_{\0}\parallel\pi_{\d }\right)\bigg\}\ge
(1-\epsilon)\left(\lambda-\mathcal{O}(t^{-d})\right).%\left(1-\frac{0.55}{e^{2}}\right)\left(\lambda-\mathcal{O}(t^{-d})\right),
\]
This shows that, in very high dimensions, 
the maximum distance term is arbitrarily close to $\lambda$ 
%in when the dimension $d$ is high is very close to $\lambda$,
which is 
%where $t \in (1,\infty)$ is some constant independent with the dimension $d$. 
%This shows that Note that 
the maximum possible value of  $\mathbb{D}_{\mathcal{F}_{[0,1]}}\left(\lambda\pi_{\0}\parallel\pi_{\d }\right)$ (see Theorem~\ref{thm:dual}). 
{
% \color{red}
In particular, this implies that in high dimensional scenario, once $\ff(\x)\le(1-\epsilon)$ for some small $\epsilon$, we have $\V(\F_{[0,1]}, ~ \B_{\ell_\infty, r}) = \mathcal{O}(t^{-d})$ and thus fail to certify.
}
\end{theorem}

\paragraph{Remark}
The condition 1) and 2)  in 
Theorem~\ref{thm: impossible linf} 
%The condition $\mathrm{mode}(\left\Vert \z\right\Vert _{\infty})=Cr$,
%$\left\Vert \z\right\Vert _{\infty}=\mathcal{O}_{p}(1)$  
 are used to ensure that the magnitude of the random perturbations generated by $\pizero$ is within a reasonable range such that the value of $\ff(\x)$ is not too small, in order to have a high accuracy  in the trade-off in \eqref{eq:objective2}. 
Note that the natural images are often contained in cube $[0,1]^{d}$. If $\norm{\z}_\infty$ is too large to exceed the region of natural images, 
%If the drawn samples are always out of %the region of natural images, which is often within $\norm{x}_\infty\leq 1$}, 
the accuracy will be obviously rather poor.
{
% \color{red}
Note that if we use variants of Gaussian distribution, we only need $||\z||_2/\sqrt{d}$ to be not too large.} 
%Therefore, 
% \red{(Why? do not see clear connection)}.
Theorem~\ref{thm: impossible linf} says that once
$\norm{\z}_\infty$ is in a reasonably small scale,   
the maximum distance term must be unreasonably large in high dimensions, yielding a vacuous lower bound.  

 %As shown in Figur~\ref{fig:linf_example},our proposed distribution can get better result than the distribution in \eqref{eq:inf_inf_distribution} and Gaussian distribution in the toy example. 

\begin{proof}
First notice that the distribution of $\z$ can be factorized by the
following hierarchical scheme: 
\begin{align*}
a & \sim\pi_{R}(a)\propto a^{d-1-k}e^{-\frac{a^{2}}{2\sigma^{2}}}\mathbb{I}\{a\ge0\}\\
\boldsymbol{s} & \sim\mathrm{Unif}^{\otimes d}(-1,1)\\
\z & \gets\frac{\boldsymbol{s}}{\left\Vert \boldsymbol{s}\right\Vert _{\infty}}a.
\end{align*}

Without loss of generality, we assume $\d^{*}=[r,...,r]^{\top}$. (see Theorem \ref{thm:opt_mu_general})

\[
\mathbb{D}_{\mathcal{F}_{[0,1]}}\left(\lambda\pi_{\0}\parallel\pi_{\d^{*}}\right)=\mathbb{E}_{\z\sim\pi_{\0}}\left(\lambda-\frac{\pi_{\d}}{\pi_{\0}}(\z)\right)_{+}.
\]
Notice that as the distribution is symmetry, 
\[
\mathrm{P}_{\pizero}\left(\left\Vert \z+\d^{*}\right\Vert _{\infty}=a+r~~\large\mid~~\left\Vert \z\right\Vert _{\infty}=a\right)=\frac{1}{2}.
\]
Define $\left|z\right|^{(i)}$ is the $i$-th order statistics of
$\left|z_{j}\right|$, $j=1,...,d$ conditioning on $\left\Vert \z\right\Vert _{\infty}=a$.
By the factorization above and some algebra, we have, for any $\epsilon \in (0,1)$,
\[
\mathrm{P}\left(\frac{\left|z\right|^{(d-1)}}{\left|z\right|^{(d)}}>(1-\epsilon)\mid\left\Vert \z\right\Vert _{\infty}=a\right)\ge1-(1-\epsilon)^{d-1}.
\]
And $\frac{\left|z\right|^{(d-1)}}{\left|z\right|^{(d)}}\perp\left|z\right|^{(d)}$.
Now we estimate $\mathbb{D}_{\mathcal{F}_{[0,1]}}\left(\lambda\pi_{\0}\parallel\pi_{\d^{*}}\right)$.
\begin{align*}
 & \mathbb{E}_{\z\sim\pi_{\0}}\left(\lambda-\frac{\pi_{\d}}{\pi_{\0}}(\z)\right)_{+}\\
= & \mathbb{E}_{a}\mathbb{E}_{\z\sim\pi_{\0}}\left[\left(\lambda-\frac{\pi_{\d}}{\pi_{\0}}(\z)\right)_{+}\mid\left\Vert \z\right\Vert _{\infty}=a\right]\\
= & \frac{1}{2}\mathbb{E}_{a}\mathbb{E}_{\z\sim\pi_{\0}}\left[\left(\lambda-\frac{\pi_{\d}}{\pi_{\0}}(\z)\right)_{+}\mid\left\Vert \z\right\Vert _{\infty}=a,\left\Vert \z+\d^{*}\right\Vert _{\infty}=a+r\right]\\
+ & \frac{1}{2}\mathbb{E}_{a}\mathbb{E}_{\z\sim\pi_{\0}}\left[\left(\lambda-\frac{\pi_{\d}}{\pi_{\0}}(\z)\right)_{+}\mid\left\Vert \z\right\Vert _{\infty}=a,\left\Vert \z+\d^{*}\right\Vert _{\infty}\neq a+r\right].
\end{align*}

Conditioning on $\left\Vert \z\right\Vert _{\infty}=a,\left\Vert \z+\d^{*}\right\Vert _{\infty}=a+r$,
we have 
\begin{align*}
\frac{\pi_{\d}}{\pi_{\0}}(\z) & =\left(\frac{1}{1+\frac{r}{a}}\right)^{k}e^{-\frac{1}{2\sigma^{2}}\left(2ra+r^{2}\right)}\\
 & =\left(\frac{1}{1+\frac{r}{a}}\right)^{k}e^{-\frac{d-1-k}{2C^{2}}\left(2\frac{a}{r}+1\right)}.
\end{align*}
Here the second equality is because we choose $\mathrm{mode}(\left\Vert \z\right\Vert _{\infty})=Cr$,
which implies that $\sqrt{d-1-k}\sigma=Cr$. And thus we have 
\begin{align*}
 & \mathbb{E}_{a}\mathbb{E}_{\z\sim\pi_{\0}}\left[\left(\lambda-\frac{\pi_{\d}}{\pi_{\0}}(\z)\right)_{+}\mid\left\Vert \z\right\Vert _{\infty}=a,\left\Vert \z+\d^{*}\right\Vert _{\infty}=a+r\right]\\
= & \int\left(\lambda-\left(\frac{1}{1+\frac{r}{a}}\right)^{k}e^{-\frac{d-1-k}{2C^{2}}\left(2\frac{a}{r}+1\right)}\right)_{+}\pi(a)da\\
= & \int\left(\lambda-\left(1+\frac{r}{a}\right)^{-k}\left(e^{\frac{2a/r+1}{2C^{2}}}\right)^{-(d-1-k)}\right)_{+}\pi(a)da\\
= & \lambda-\mathcal{O}(t^{-d}),
\end{align*}
for some $t>1$. Here the last equality is by the assumption that
$\left\Vert \z\right\Vert _{\infty}=\mathcal{O}_{p}(1)$.

Next we bound the second term $\mathbb{E}_{a}\mathbb{E}_{\z\sim\pi_{\0}}\left[\left(\lambda-\frac{\pi_{\d}}{\pi_{\0}}(\z)\right)_{+}\mid\left\Vert \z\right\Vert _{\infty}=a,\left\Vert \z+\d^{*}\right\Vert _{\infty}\neq a+r\right]$.
By the property of uniform distribution, we have 
\begin{align*}
 & \mathrm{P}\left(\frac{\left|z\right|^{(d-1)}}{\left|z\right|^{(d)}}>(1-\epsilon)\mid\left\Vert \z\right\Vert _{\infty}=a,\left\Vert \z+\d^{*}\right\Vert _{\infty}\neq a+r\right)\\
= & \mathrm{P}\left(\frac{\left|z\right|^{(d-1)}}{\left|z\right|^{(d)}}>(1-\epsilon)\mid\left\Vert \z\right\Vert _{\infty}=a\right)\\
\ge & 1-(1-\epsilon)^{d-1}.
\end{align*}
And thus, for any $\epsilon\in[0,1)$, 
\[
\mathrm{P}\left(\left\Vert \z+\d^{*}\right\Vert _{\infty}\ge\left((1-\epsilon)a+r\right)^{2}\mid\left\Vert \z\right\Vert _{\infty}=a,\left\Vert \z+\d^{*}\right\Vert _{\infty}\neq a+r\right)\ge\frac{1}{2}\left(1-(1-\epsilon)^{d-1}\right).
\]
It implies that 
\begin{align*}
 & \mathbb{E}_{\z\sim\pi_{\0}}\left[\left(\lambda-\frac{\pi_{\d}}{\pi_{\0}}(\z)\right)_{+}\mid\left\Vert \z\right\Vert _{\infty}=a,\left\Vert \z+\d^{*}\right\Vert _{\infty}=a+r\right]\\
\ge & \frac{1}{2}\left(1-(1-\epsilon)^{d-1}\right)\left(\lambda-\left(1-\epsilon+\frac{r}{a}\right)^{-k}e^{-\frac{1}{2\sigma^{2}}\left(\epsilon(\epsilon-2)a^{2}+2r(1-\epsilon)a+r^{2}\right)}\right)_{+}\\
= & \frac{1}{2}\left(1-(1-\epsilon)^{d-1}\right)\left(\lambda-\left(1-\epsilon+\frac{r}{a}\right)^{-k}e^{-\frac{d-1-k}{2C^{2}}\left(\epsilon(\epsilon-2)a^{2}/r^{2}+2(1-\epsilon)a/r+1\right)}\right)_{+}.
\end{align*}
For any $\epsilon'\in(0,1)$, by choosing $\epsilon=\frac{\log(2/\epsilon')}{d-1}$,
for large enough $d$, we have 
\begin{align*}
 & \mathbb{E}_{\z\sim\pi_{\0}}\left[\left(\lambda-\frac{\pi_{\d}}{\pi_{\0}}(\z)\right)_{+}\mid\left\Vert \z\right\Vert _{\infty}=a,\left\Vert \z+\d^{*}\right\Vert _{\infty}=a+r\right]\\
\ge & \frac{1}{2}\left(1-(1-\epsilon)^{d-1}\right)\left(\lambda-\left(1-\epsilon+\frac{r}{a}\right)^{-k}e^{-\frac{d-1-k}{2C^{2}}\left(2(1-\epsilon)a/r+1\right)}e^{\frac{a^{2}\log(2/\epsilon')}{C^{2}r^{2}}}\right)_{+}\\
= & \frac{1}{2}\left(1-(1-\frac{\log(2/\epsilon')}{d-1})^{d-1}\right)\left(\lambda-\left(1-\frac{\log(2/\epsilon')}{d-1}+\frac{r}{a}\right)^{-k}e^{-\frac{d-1-k}{2C^{2}}\left(2(1-\epsilon)a/r+1\right)}e^{\frac{a^{2}\log(2/\epsilon')}{C^{2}r^{2}}}\right)_{+}\\
\ge & \frac{1}{2}\left(1-\epsilon'\right)\left(\lambda-\left(1-\epsilon+\frac{r}{a}\right)^{-k}e^{-\frac{d-1-k}{2C^{2}}\left(2(1-\epsilon)a/r+1\right)}e^{\frac{a^{2}\log(2/\epsilon')}{C^{2}r^{2}}}\right)_{+}.
\end{align*}
Thus we have 
\begin{align*}
 & \frac{1}{2}\mathbb{E}_{a}\mathbb{E}_{\z\sim\pi_{\0}}\left[\left(\lambda-\frac{\pi_{\d}}{\pi_{\0}}(\z)\right)_{+}\mid\left\Vert \z\right\Vert _{\infty}=a,\left\Vert \z+\d^{*}\right\Vert _{\infty}\neq a+r\right]\\
= & \frac{1}{2}\left(1-\epsilon'\right)\left(\lambda-\mathcal{O}(t^{-d})\right).
\end{align*}
Combine the bounds, for large $d$, we have 
\[
\mathbb{D}_{\mathcal{F}_{[0,1]}}\left(\lambda\pi_{\0}\parallel\pi_{\d^{*}}\right)=\left(1-\epsilon'\right)\left(\lambda-\mathcal{O}(t^{-d})\right).
\]

\end{proof}

\section{More about Experiments}
\label{section:algorithm}

\subsection{Practical Algorithm}
In this section, we give our algorithm for certification. Our target is to give a high probability bound for the solution of
\[
\V(\F_{[0,1]}, ~ \B_{\ell_\infty, r}) = \max_{\lambda\ge 0}\left\{\lambda \ff-\mathbb{D}_{\mathcal{F}_{[0,1]}}\left(\lambda\pi_{\0}\parallel\pi_{\d }\right)\right\}
\]
given some classifier $\ftrue$. Following \cite{cohen2019certified}, the given classifier here has a binary output $\{0, 1\}$. Computing the above quantity requires us to evaluate both $\ff$ and $\mathbb{D}_{\mathcal{F}_{[0,1]}}\left(\lambda\pi_{\0}\parallel\pi_{\d }\right)$. A lower bound $\hat p_0 $ of the former term is obtained through binominal test as \cite{cohen2019certified} do, while the second term can be estimated with arbitrary accuracy using Monte Carlo samples. We perform grid search to optimize $\lambda$ and given $\lambda$, we draw $N$ $\mathrm{i.i.d.}$ samples from the proposed smoothing distribution $\pi_{\0}$ to estimate $\lambda \ff-\mathbb{D}_{\mathcal{F}_{[0,1]}}\left(\lambda\pi_{\0}\parallel\pi_{\d }\right)$. This can be achieved by the following importance sampling manner:
\begin{eqnarray*}
&&\lambda\ff-\mathbb{D}_{\mathcal{F}_{[0,1]}}\left(\lambda\pi_{\0}\parallel\pi_{\d}\right)
\\
&\ge&\lambda \hat p_0 -\int\left(\lambda-\frac{\pi_{\d}}{\pi_{\0}}(\z)\right)_{+}\pi_{\0}(\z)d\z
\\
&\ge &\lambda \hat p_0 - \frac{1}{N}\sum_{i=1}^{N}\left(\lambda-\frac{\pi_{\d}}{\pi_{\0}}(\z_{i})\right)_{+} - \epsilon.
\end{eqnarray*}
And we use reject sampling to obtain samples from $\pi_{\0}$. Notice that, we restrict the search space of $\lambda$ to a finite compact set so the importance samples is bounded. Since the Monte Carlo estimation is not exact with an error $\epsilon$, we give a high probability concentration lower bound of the estimator. Algorithm \ref{alg:1} summarized our algorithm.

\begin{algorithm}[H]
\SetAlgoLined
\KwIn{input image $\x$; 
original classifier: $f^{\sharp}$;  
smoothing distribution $\pi_{\0}$; 
  radius $r$; 
  search interval $[\lambda_{\text{start}}, \lambda_{\text{end}}]$  of $\lambda$; search precision $h$ for optimizing $\lambda$; number of samples $N_1$ for testing $p_0$;
  pre-defined error threshold $\epsilon$;
  significant level $\alpha$; 
  }
% \KwOut{Write here the result}
%  initialization\;
~~compute search space for $\lambda$ : $\Lambda= $range$(\lambda_{\text{start}}, \lambda_{\text{end}}, h)$\; \\
compute $N_2$: number of Monte Carlo estimation given $\epsilon, \alpha$ and $\Lambda$ \; \\
compute optimal disturb: $\d$
depends on specific setting\; \\
 \For{$\lambda$ $\text{in}$ $\Lambda$}{
    sample $\z_1, \cdots, \z_{N_1} \sim \pi_{\0}$\; \\
    compute $n_1 = \frac{1}{N_1}\sum_{i=1}^{N_1}f^\sharp (\x +\z_i)$\; \\
    compute $\hat p_0 =$LowerConfBound$(n_1, N_1, 1-\alpha)$ \; \\
     sample $\z_1, \cdots, \z_{N_2} \sim \pi_{\0}$\; \\
    compute $\hat\D_{\F_{[0,1]}}\left(\text{\ensuremath{\lambda}}\pi_{\0} ~\Vert~ \pi_{\d} \right) = \frac{1}{N_2}\sum_{i=1}^{N_2}\left(\text{\ensuremath{\lambda}}-\frac{\pi_{\d}}{\pi_{\0}}(\z_{i})\right)_{+}$\; \\
    % compute an error bar estimate $\epsilon$ from $N_1, N_2, \lambda$ \; \\
    compute confidence lower bound $b_\lambda = \lambda\hat{p}_{0} - \hat\D_{\F_{[0,1]}}\left(\text{\ensuremath{\lambda}}\pi_{\0} ~\Vert~ \pi_{\d} \right) - \epsilon$
 }
 \uIf{$\max_{\lambda\in\Lambda}b_\lambda \ge 1/2$}
 {$\x$ can be certified}
 \Else{$\x$ cannot be certified}
 \caption{Certification algorithm}
 \label{alg:1}
\end{algorithm}

The LowerConfBound function performs a binominal test as described in \cite{cohen2019certified}.
The $\epsilon$ in Algorithm \ref{alg:1} is given by concentration inequality.

\begin{theorem}
Let $h(z_1, \cdots, z_N) = \frac{1}{N}\sum_{i=1}^N\left( \lambda - \frac{\pi_{\d}(z_i)}{\pi_{\0}(z_i)}\right)_{+}$, 
we yield
$$
\operatorname { Pr } \{| h(z_1, \cdots, z_N) - \int \left(\lambda\pi_{\0}(z) - \pi_{\d}(z)\right)_{+} d z | \geq \varepsilon \}
\leq \exp \left(  \frac {- 2 N \varepsilon^2} {\lambda^2} \right).
$$
\end{theorem}

\begin{proof}
Given \emph{McDiarmid's Inequality}, which says
$$
\sup _ { x _ { 1 } , x _ { 2 } , \ldots , x _ { n } , \hat { x } _ { i } } \left| h \left( x _ { 1 } , x _ { 2 } , \ldots , x _ { n } \right) - h \left( x _ { 1 } , x _ { 2 } , \ldots , x _ { i - 1 } , \hat { x } _ { i } , x _ { i + 1 } , \ldots , x _ { n } \right) \right| \leq c _ { i } \quad \text { for } \quad 1 \leq i \leq n,
$$
we have $c_i = \frac{\lambda}{N}$, and then obtain
$$
\operatorname { Pr } \{| h(z_1, \cdots, z_N) - \int \left(\lambda\pi_{\0}(z) - \pi_{\d}(z)\right)_{+} d z | \geq \varepsilon \}
\leq \exp \left(  \frac {- 2 N \varepsilon^2} {\lambda^2} \right).
$$
\end{proof}

The above theorem tells us that, 
once $\epsilon, \lambda, N$ is given,
we can yield a bound with high-probability $1-\alpha$. One can also get $N$ when $\epsilon, \lambda, \alpha$ is provided. Note that this is the same as the Hoeffding bound mentioned in Section \ref{sec:L2} as Micdiarmid bound is a generalization of Hoeffding bound.

However, in practice we can use a small trick as below to certify with much less comupation:
\begin{algorithm}[H]
\SetAlgoLined
%\KwIn{image: $\x$, smoothing distribution: $\pi_{\0}, \pi_{\d}$, given classifier: $f^{\sharp}$, radius: $r$, search interval for $\lambda$: $[\lambda_{\text{start}}, \lambda_{\text{end}}]$, search precision: $h$, number of Monte Carlo for first estimation: $N_1^0, N^0_2$, pre-defined error: $\epsilon$, $\alpha$: significant level}
\KwIn{input image $\x$; 
original classifier: $f^{\sharp}$;  
smoothing distribution $\pi_{\0}$; 
  radius $r$; 
  search interval for $\lambda$: $[\lambda_{\text{start}}, \lambda_{\text{end}}]$; search precision $h$ for optimizing $\lambda$; number of Monte Carlo for first estimation: $N_1^0, N^0_2$; number of samples $N_1$ for a second test of $p_0$;
  pre-defined error threshold $\epsilon$; 
  significant level $\alpha$; 
  optimal perturbation $\d$ ($\d=[r,0,\ldots,0]^\top$ for $\ell_2$ attacking and $\d=[r,\ldots, r]^\top$ for $\ell_\infty$ attacking).
  }
 %compute search space for $\lambda$ : $\Lambda=$range$(\lambda_{\text{start}}, \lambda_{\text{end}}, h)$\; \\
% depends on specific setting\; \\
 \For{$\lambda$ $\text{in}$ $\Lambda$}{
    sample $\z_1, \cdots, \z_{N^0_1} \sim \pi_{\0}$\; \\
    compute $n^0_1 = \frac{1}{N^0_1}\sum_{i=1}^{N^0_1}f^\sharp (\x +\z_i)$\; \\
    compute $\hat p_0 =$LowerConfBound$(n^0_1, N^0_1, 1-\alpha)$ \; \\
    sample $\z_1, \cdots, \z_{N^0_2} \sim \pi_{\0}$\; \\
    compute $\hat\D_{\F_{[0,1]}}\left(\text{\ensuremath{\lambda}}\pi_{\0} ~\Vert~ \pi_{\d} \right) = \frac{1}{N^0_2}\sum_{i=1}^{N^0_2}\left(\text{\ensuremath{\lambda}}-\frac{\pi_{\d}}{\pi_{\0}}(\z_{i})\right)_{+}$\; \\
    compute confidence lower bound $b_\lambda = \lambda\hat{p}_{0} - \hat\D_{\F_{[0,1]}}\left(\text{\ensuremath{\lambda}}\pi_{\0} ~\Vert~ \pi_{\d} \right)$
 }
 ~~compute $\hat\lambda = \arg\max_{\lambda \in\Lambda} b_\lambda$ \; \\
 compute $N_2$:  number of Monte Carlo estimation given $\epsilon, \alpha$ and $\hat\lambda$ \; \\
  sample $\z_1, \cdots, \z_{N_1} \sim \pi_{\0}$\; \\
    compute $n_1 = \frac{1}{N_1}\sum_{i=1}^{N_1}f^\sharp (\x +\z_i)$\; \\
    compute $\hat p_0 =$LowerConfBound$(n_1, N_1, 1-\alpha)$ \; \\
     sample $\z_1, \cdots, \z_{N_2} \sim \pi_{\0}$\; \\
    compute $\hat\D_{\F_{[0,1]}}\left(\text{\ensuremath{\lambda}}\pi_{\0} ~\Vert~ \pi_{\d} \right) = \frac{1}{N_2}\sum_{i=1}^{N_2}\left(\text{\ensuremath{\lambda}}-\frac{\pi_{\d}}{\pi_{\0}}(\z_{i})\right)_{+}$\; \\
 compute $b = \hat\lambda \hat{p}_{0} - \hat\D_{\F_{[0,1]}}\left(\text{\ensuremath{\lambda}}\pi_{\0} ~\Vert~ \pi_{\d} \right) - \epsilon$ \; \\
 \uIf{$b \ge 1/2$}
 {$\x$ can be certified}
 \Else{$\x$ cannot be certified}
 \caption{Practical certification algorithm}
 \label{alg:2}
\end{algorithm}
Algorithm \ref{alg:2} allow one to begin with small $N_1^0, N_2^0$ to obtain the first estimation and choose a $\hat\lambda$. Then a rigorous lower bound can be achieved with $\hat\lambda$ with enough (i.e. $N_1, N_2$) Monte Carlo samples.

\subsection{Experiment Settings}
\label{sec: hyperparameters}
The details of our method are shown in the supplementary material. 
Since our method requires  Monte Carlo approximation, 
 we draw $0.1M$ samples from $\pizero$ and construct 
 $\alpha=$ 99.9\% confidence lower bounds of that in \eqref{eq:objective2}. The optimization on $\lambda$ is solved using grid search. 
For $\ell_2$ attacks, 
we set $k = 500$ for CIFAR-10 and $k = 50000$ for ImageNet in our non-Gaussian smoothing distribution \eqref{equ:L2smooth}. 
If the used model was trained with a Gaussian perturbation noise of  $\mathcal{N}(0, \sigma_0^2)$,
then the $\sigma$ parameter of our smoothing distribution  
is set to be $\sqrt{({d - 1})/({d - 1 - k})}\sigma_0$,
such that the expectation of the norm $\norm{\z}_2$ under our non-Gaussian distribution \eqref{equ:L2smooth} matches with the norm of $\normal(0, \sigma_0^2)$. 
For $\ell_1$ situation, we keep the same rule for hyperparameter selection as $\ell_2$ case, in order to make the norm of proposed distribution has the same mean with original distribution. 
% The certification results of baseline might be slightly different from \cite{teng2020ell} in which the whole dataset is assessed while we follow \cite{cohen2019certified} to certify a subset with 500 data points.
For $\ell_\infty$ situation, we set $k = 250$ and $\sigma$ also equals to $\sqrt{{(d - 1})/({d - 1 - k}})\sigma_0$ for the mixed norm smoothing distribution \eqref{equ:mixednormfamily} just for consistency. 
% In all cases, the baseline algorithm uses a Gaussian smoothing distribution $\normal(0, \sigma_0^2)$. 
More ablation study about $k$ is deferred to Appendix~\ref{sec:details}.

% \subsection{Clarification about $\ell_\infty$ Experiments}\label{sec:l_inf_clarification}
% We choose \cite{salman2019provably} instead of \cite{cohen2019certified} as $\ell_\infty$ baseline  because its model performs much better for $\ell_\infty$ certification than \cite{cohen2019certified} and thus is more proper for demonstration. 
% For example, when $r=8/255$, we improve \cite{salman2019provably} from 25\% to 32\% and improve \cite{cohen2019certified} from 12\% to 15\%. When $r=12/255$, we improve \cite{salman2019provably} from 13\% to 17\% and improve \cite{cohen2019certified} from 4\% to 6\%.
% This is probably result from that it's trained adversarially, which matches $\ell_\infty$ space property.

% \red{For Chengyue here}

\subsection{Abalation Study}
\label{sec:details}

On CIFAR10, we also do ablation study to show the influence of different $k$ for the $\ell_2$ certification case as shown in Table~\ref{table:l2-cifar10-ablation}.

\begin{table}[!htbp]
	\begin{center}
		\begin{tabular}{l|ccccccccc}
			\hline
            $\ell_2$ Radius & $0.25$ & $0.5$ & $0.75$ & $1.0$ & $1.25$ & $1.5$ & $1.75$ & $2.0$ & $2.25$ \\
            \hline
            Baseline (\%) & 60 & 43 & 34 & 23 & 17 & 14 & 12 & 10 & 8 \\
            $k = 100$ (\%) & 60 & 43 & 34 & 23 & 18 & 15 & 12 & 10 & 8 \\
            $k = 200$ (\%) & 60 & 44 & 36 & 24 & 18 & 15 & 13 & 10 & 8 \\
            $k = 500$ (\%) & \bf{61} & \bf{46} & \bf{37} & \bf{25} & \bf{19} & \bf{16} & 14 & 11 & \bf{9} \\
            $k = 1000$ (\%) & 59 & 44 & 36 & \bf{25} & \bf{19} & \bf{16} & 14 & 11 & \bf{9} \\
            $k = 2000$ (\%) & 56 & 41 & 35 & 24 & \bf{19} & \bf{16} & \bf{15} & \bf{12} & \bf{9} \\
            \hline
		\end{tabular}
	\end{center}
	\caption{\label{table:l2-cifar10-ablation} Certified top-1 accuracy of the best classifiers on cifar10 at various $\ell_2$ radius. We use the same model as \cite{cohen2019certified} and do not train any new models. }
\end{table}

\section{Illumination about Bilateral Condition\protect\footnote{In fact, the theoretical part of 
\cite{jia2020certified} 
share some similar discussion with this section.}}
\label{sec:bilateral}
The results in the main context is obtained under binary classfication setting. Here we show it has a natural generalization to multi-class classification setting. Suppose the given classifier $\ft$ classifies an input $\x$ correctly to class A, i.e.,
\begin{equation}
\label{eq:multi-correct}
\ft_A(\x) > \max_{B \ne A} \ft_B(\x)    
\end{equation}
where $\ft_B(\x)$ denotes the prediction confidence of any class $B$ different from ground truth label $A$. Notice that $\ft_A(\x) + \sum_{B\ne A}\ft_B(\x)=1$, so the necessary and sufficient condition for correct binary classification $\ft_A(\x) > 1/2$ becomes a \emph{sufficient} condition for multi-class prediction. 

Similarly, the necessary and sufficient condition for correct classification of the \emph{smoothed} classifier is 
\begin{align*}
   \min_{f\in\F}  & \bigg \{ \E_{\z\sim\pizero}[f_A(\x+\d+\z)]
~~~~~~\mathrm{s.t.}~~~~~\E_{\pizero}[f_A(\x)] = \fpizero_{,A}(\x) \bigg\}
 > \\ 
\max_{f\in\F} & \bigg \{ \E_{\z\sim\pizero}[f_B(\x+\d+\z)] 
~~~~~~\mathrm{s.t.}~~~~~\E_{\pizero}[f_B(\x)] = \fpizero_{,B}(\x) \bigg\}
\end{align*}
for $\forall B \ne A$ and any perturbation $\d\in\B$. Writing out their Langragian forms makes things clear:
\begin{align*}
    \max_{\lambda} \lambda \fpizero_{,A}(\x) - \D_{\F_{[0,1]}}\left(\text{\ensuremath{\lambda}}\pi_{\0} ~\Vert~ \pi_{\d} \right) > \min_{\lambda}\max_{B\ne A} \lambda \fpizero_{, B}(\x) + \D_{\F_{[0,1]}}\left(\pi_{\d}~\Vert~\text{\ensuremath{\lambda}}\pi_{\0} \right)
\end{align*}
Thus the overall necessary and sufficient condition is 
\begin{align*}
    \min_{\d\in\B}\left\{\max_{\lambda} \left(\lambda \fpizero_{,A}(\x) - \D_{\F_{[0,1]}}\left(\text{\ensuremath{\lambda}}\pi_{\0} ~\Vert~ \pi_{\d} \right) \right) - \max_{B\ne A}\min_{\lambda}\left( \lambda \fpizero_{, B}(\x) + \D_{\F_{[0,1]}}\left(\pi_{\d}~\Vert~\text{\ensuremath{\lambda}}\pi_{\0} \right)\right) \right\} > 0
\end{align*}
Optimizing this \emph{bilateral} object will \emph{theoretically give a better certification result} than our method in main context, especially when the number of classes is large. But we do not use this bilateral formulation as reasons stated below.

When both $\pi_{\0}$ and $\pi_{\d}$ are gaussian, which is \cite{cohen2019certified}'s setting, this condition is equivalent to:
\begin{align*}
    &\min_{\d\in\B}\left\{\Phi\left(\Phi^{-1}(\fpizero_{,A}(\x)) - \frac{\Vert\d\Vert_2}{\sigma}\right) - \max_{B\ne A} \Phi\left(\Phi^{-1}(\fpizero_{,B}(\x)) + \frac{\Vert\d\Vert_2}{\sigma}\right) \right\} > 0 \\
    \Leftrightarrow & \quad  \Phi^{-1}(\fpizero_{,A}(\x)) - \frac{r}{\sigma} > \Phi^{-1}(\fpizero_{,B}(\x)) + \frac{r}{\sigma}, \quad \forall B \ne A \\
    \Leftrightarrow & \quad r <
    \frac{\sigma}{2}\left(\Phi^{-1}(\fpizero_{,A}(\x)) -
    \Phi^{-1}(\fpizero_{,B}(\x))\right),\forall B \ne A
\end{align*}
with a similar derivation process like Appendix \ref{sec: retrievecohen}. This is exactly the same bound in the (restated) theorem 1 of \cite{cohen2019certified}.

\cite{cohen2019certified} use $1 - \underline{p_{A}}$ as a naive estimate of the upper bound of $\fpizero_{,B}(\x)$, where $\underline{p_{A}}$ is a lower bound of $\fpizero_{,A}(\x)$. This leads the confidence bound decay to the bound one can get in binary case, i.e., $r \le \sigma\Phi^{-1}(\fpizero_{,A}(\x)) $.

% If we are to use this bilateral certification method, 
% for fair comparison, 
% we need to compare our bound with Cohen's bilateral version.
% It means using another binominal test to get a upper bound of $\fpizero_{,B}(\x)$. However, two important baselines \cite{cohen2019certified, salman2019provably} do not take the bilateral form. Thus for a succinct comparison we do not choose this bilateral formulation in this paper.
As the two important baselines \cite{cohen2019certified, salman2019provably} do not take the bilateral form, 
we also do not use this form in experiments for fairness.

% This document was modified from the file originally made available by
% Pat Langley and Andrea Danyluk for ICML-2K. This version was created
% by Iain Murray in 2018, and modified by Alexandre Bouchard in
% 2019 and 2020. Previous contributors include Dan Roy, Lise Getoor and Tobias
% Scheffer, which was slightly modified from the 2010 version by
% Thorsten Joachims & Johannes Fuernkranz, slightly modified from the
% 2009 version by Kiri Wagstaff and Sam Roweis's 2008 version, which is
% slightly modified from Prasad Tadepalli's 2007 version which is a
% lightly changed version of the previous year's version by Andrew
% Moore, which was in turn edited from those of Kristian Kersting and
% Codrina Lauth. Alex Smola contributed to the algorithmic style files.

\end{document}